\documentclass[a4paper]{article}
\usepackage{hyperref}
\usepackage{pdfpages}

\usepackage[utf8]{inputenc} 
\usepackage[T1]{fontenc}    
\usepackage{bm}
\usepackage{hyperref}       
\usepackage{url}            
\usepackage{booktabs}       
\usepackage{amsfonts}       
\usepackage{nicefrac}       
\usepackage{microtype}      
\usepackage{amsthm}
\usepackage{amssymb, amsmath}
\usepackage{graphicx}
\usepackage{multicol}
\usepackage{multirow}
\usepackage{subfigure}
\usepackage{CJKutf8}
\usepackage{xcolor}  
\usepackage[numbers]{natbib}
\newcommand{\argmax}{\mathop{\mathrm{arg~max}}\limits}
\newcommand{\argmin}{\mathop{\mathrm{arg~min}}\limits}

\newtheorem{assumption}{Assumption}

\newtheorem{thm}{Theorem}
\newtheorem{lmm}{Lemma}
\newtheorem{col}{Corollary}
\newtheorem{rmk}{Remark}

\usepackage[acronym,nomain]{glossaries}

\usepackage[acronym,nomain]{glossaries}
\usepackage{authblk}
\newglossary[slg]{symbolslist}{syi}{syg}{Symbolslist}
\makeglossaries

\makeatletter
\newcommand{\figcaption}[1]{\def\@captype{figure}\caption{#1}}
\newcommand{\tblcaption}[1]{\def\@captype{table}\caption{#1}}
\makeatother

\title{Excess risk analysis for epistemic uncertainty with application to variational inference}

\author[1]{\href{mailto:Futoshi Futami <futoshi.futami.uk@hco.ntt.co.jp>?Subject=Archiv paper}{Futoshi Futami\thanks{futami.futoshi.es@osaka-u.ac.jp}}{}} 
\author[2]{Tomoharu Iwata}
\author[2]{Naonori Ueda}
\author[3]{Issei Sato}
\author[3]{Masashi Sugiyama}
\affil[1]{%
    Osaka university\\
    Osaka, Japan
}
\affil[2]{%
    Communication Science Laboratories\\
    NTT\\
    Kyoto, Japan
}
\affil[3]{%
    The University of Tokyo\\
    Tokyo, Japan
}

\begin{document}

\maketitle

\begin{abstract}
Bayesian deep learning plays an important role especially for its ability evaluating epistemic uncertainty (EU). Due to computational complexity issues, approximation methods such as variational inference (VI) have been used in practice to obtain posterior distributions and their generalization abilities have been analyzed extensively, for example, by PAC-Bayesian theory; however, little analysis exists on EU, although many numerical experiments have been conducted on it. 
In this study, we analyze the EU of supervised learning in approximate Bayesian inference by focusing on its excess risk. First, we theoretically show the novel relations between generalization error and the widely used EU measurements, such as the variance and mutual information of predictive distribution, and derive their convergence behaviors. Next, we clarify how the objective function of VI regularizes the EU. With this analysis, we propose a new objective function for VI that directly controls the prediction performance and the EU based on the PAC-Bayesian theory. Numerical experiments show that our algorithm significantly improves the EU evaluation over the existing VI methods.
\end{abstract}

\section{Introduction}\label{ch:intro}
As machine learning applications spread, understanding the uncertainty of predictions is becoming more important to increase our confidence in machine learning algorithms \citep{bhatt2021uncertainty}. Uncertainty refers to the variability of a prediction caused by missing information. For example, in regression problems, it corresponds to the error bars in predictions; and in classification problems, it is often expressed as the class posterior probability,  entropy, and mutual information  \citep{hullermeier2021aleatoric,gawlikowski2022survey}. There are two types of uncertainty \citep{bhatt2021uncertainty}: 1) Aleatoric uncertainty (AU), which is caused by noise in the data itself, and 2) Epistemic uncertainty (EU), which is caused by a lack of training data.  In particular, since EU can tell us where in the input space is yet to be learned, integrated with deep learning methods, it is used in such applications as dataset shift \citep{NEURIPS2019_8558cb40}, adversarial data detection \citep{NEURIPS2018_586f9b40}, active learning \citep{houlsby2011bayesian}, Bayesian optimization \citep{hernandez2014predictive}, and reinforcement learning \citep{janz2019successor}.

 Mathematically, AU is defined as Bayes risk, which expresses the fundamental difficulty of learning problems \citep{depeweg2018decomposition,jain2021deup,xu2020continuity}. For EU, Bayesian inference is useful because posterior distribution updated from prior distribution can represent a lack of data \citep{hullermeier2021aleatoric}. In practice, measurements like the variance of the posterior predictive distribution, and associated conditional mutual information represented EU in practice \citep{kendall2017uncertainties,depeweg2018decomposition}.
 
 In Bayesian inference, since posterior distribution is characterized by the training data and the model using Bayes' formula, its prediction performance and EU are determined automatically. However, due to computational issues, such exact Bayesian inference is difficult to implement; we often use approximation methods, such as variational inference (VI) \citep{PRML}, especially for deep Bayesian models. Since the derived posterior distribution also depends on the properties of approximation methods, the prediction performance and EU of deep Bayesian learning are no longer automatically guaranteed through Bayes' formula. The prediction performance has been analyzed as generalization error, for example, by PAC-Bayesian theory \citep{alquier2021user}. Since EU is also essential in practice, we must obtain a theoretical guarantee of the algorithm- and the sample- dependent non-asymptotic theory for EU, similarly to generalization error analysis.

 Unfortunately, study has been limited in that direction. Traditional EU analysis has focused on the properties of the exact Bayesian posterior and predictive distributions \citep{fiedler2021practical,lederer2019posterior} as well as large sample behaviors \citep{54897}. Since Bayesian deep learning uses approximate posterior distributions, we cannot apply such traditional EU analysis based on Bayes' formula to Bayesian deep learning. The asymptotic theory of a sufficiently large sample may overlook an important property of the EU that is due to the lack of training data.

Recently, analysis of EU focusing on loss functions was proposed for supervised learning \citep{xu2020minimum,jain2021deup}. EU was defined as the {\bf excess risk} obtained by subtracting the Bayes risk corresponding to the AU from the total risk. Thus, excess risk implies the loss due to insufficient data when the model is well specified. Although this approach successfully defines EU with loss functions, the following limitation still remains. \cite{xu2020minimum} assume that the data generating mechanism is already known and that we can precisely evaluate Bayesian posterior and predictive distribution. A correct model is not necessarily a realistic assumption, and the assumption about an exact Bayesian posterior hampers understanding EU in approximation methods.


To address these limitations, it appears reasonable to analyze excess risk under a similar setting as PAC-Bayesian theory and apply it to EU of approximate Bayesian inference.
However, as shown in Sec.~\ref{sec:Background}, analyzing excess risk in such a way leads to impractical theoretical results, and the relations between excess risk and the widely used EU measurements remain unclear. This greatly complicates EU analysis. Because of this difficulty, to the best of our knowledge, no research exists on excess risk for EU for approximation methods.

In this paper, we propose a new theoretical analysis for EU that addresses the above limitations of these existing settings. 
 Our contributions are the followings:
\begin{itemize}
    \item We show non-asymptotic analysis for widely used EU measurements (Theorems~\ref{thm_square} and \ref{thm_freq_excess_risk}). 
    We propose computing the Bayesian excess risk (BER) (Eq.\eqref{excss_bayes}) and show that this excess risk equals to widely used EU measurements.  Then we theoretically show the convergence behavior of BER using PAC-Bayesian theory (Eqs.~\eqref{eq_bound_square2}) and ~\eqref{general2}). 
    \item Based on theoretical analysis, we give a new interpretation of the existing VI that clarifies how the EU is regularized (Eqs.\eqref{gp} and \eqref{eq_mutual}). Then we propose a novel algorithm that directly controls the prediction and the EU estimation performance simultaneously based on PAC-Bayesian theory (Eq.\eqref{proposal}). Numerical experiments suggest that our algorithm significantly improves EU evaluation over the existing VI.
\end{itemize}

\section{Background of PAC-Bayesian theory and epistemic uncertainty}\label{sec:Background}
Here we introduce preliminaries. Such capital letters as $X$ represent random variables, and such lowercase letters as $x$ represent deterministic values. All the notations are summarized in Appendix~\ref{summation}.
\subsection{PAC-Bayesian theory}\label{sec_freqe}
 We consider a supervised setting and denote input-output pairs by $Z=(X,Y)\in\mathcal{Z}:=\mathcal{X}\times\mathcal{Y}$. We assume that all the data are i.i.d.~from some unknown data-generating distribution $\nu(Z)=\nu(Y|X)\nu(X)$.
 Learners can access $N$ training data, $\bold{Z}^N:=(Z_1,\dots,Z_N)$ with $Z_n:=(X_n,Y_n)$, which are generated by $\bold{Z}^N\sim \nu(Z)^{N}$. We express $\nu(Z)^{N}$ as $\nu(\bold{Z}^N)$. We express conditional distribution as $\nu(Y|X=x)$ as $\nu(Y|x)$ for simplicity. We introduce loss function $l:\mathcal{Y}\times \mathcal{A}\to\mathbb{R}$ where $\mathcal{A}$ is an action space. We express loss of action $a\in\mathcal{A}$ and target variable $y$ is written as $l(y,a)$. We introduce a model $f_\theta:\mathcal{X}\to\mathcal{A}$, parameterized by  $\theta\in\Theta\subset\mathbb{R}^d$. When we put a prior $p(\theta)$ over $\theta$, the PAC-Bayesian theory \citep{alquier2021user,germain2016pac} guarantees the prediction performance by focusing on the average of the loss with respect to posterior distribution $q(\theta|\bold{Z}^N)\in\mathcal{Q}$. $\mathcal{Q}$ is a family of distributions and $q(\theta|\bold{Z}^N)$ is not restricted to Bayesian posterior distribution. In this work we consider the log loss and the squared loss. For the log loss, we consider model $p(y|x,\theta)$, and the loss is given as $l(y,p(y|x,\theta))=-\ln p(y|x,\theta)$, where $\mathcal{A}$ is probability distributions. 
 For the squared loss, we use model $f_\theta(x)$ and $l(y, f_\theta(x))=|y-f_\theta(x)|^2$, where $\mathcal{Y}=\mathcal{A}=\mathbb{R}$.

PAC-Bayesian theory provides a guarantee for the generalization of test error $\displaystyle R^l(Y|X,\bold{Z}^N):=\mathbb{E}_{\nu(\bold{Z}^N)}\mathbb{E}_{q(\theta|\bold{Z}^N)}\mathbb{E}_{\nu(Z)}l(Y,f_\theta(X))$ and training error $\mathbb{E}_{\nu(\bold{Z}^N)}r^l(\bold{Z}^N):=\mathbb{E}_{\nu(\bold{Z}^N)}\mathbb{E}_{q(\theta|\bold{Z}^N)}\frac{1}{N}\sum_{n=1}^N l(Y_n,f_\theta(X_n))$. A typical PAC-Bayesian error bound takes form $R^l(Y|X,\bold{Z}^N)\leq  \mathbb{E}_{\nu(\bold{Z}^N)}r^l(\bold{Z}^N)+\mathrm{Gen}^l(\bold{Z}^N)$. $\mathrm{Gen}^l(\bold{Z}^N)$ is called {\bf generalization error}. Under suitable assumptions \citep{alquier2021user}, $\mathrm{Gen}^l(\bold{Z}^N)$ is upper-bounded by $\mathcal{O}(1/N^\alpha)$ for $\alpha\in(1/2,1]$.  In many cases, it depends on the complexity of the posterior distribution, such as Kullback-Leibler (KL) divergence $\mathrm{KL}(q(\theta|\bold{Z}^N)|p(\theta))$. When $\mathrm{Gen}^l(\bold{Z}^N)=\mathrm{KL}(q(\theta|\bold{Z}^N)|p(\theta))/\lambda+\mathrm{c}$, where $\lambda$ and $c$ are positive constants, given training data $\bold{Z}^N=\bold{z}^N$, we get a posterior distribution for the prediction by 
\begin{align}
\scalebox{0.95}{$\displaystyle\hat{q}(\theta|\bold{z}^N)=\argmin_{q(\theta|\bold{z}^N)\in\mathcal{Q}}r(\bold{z}^N)+\frac{\mathrm{KL}(q(\theta|\bold{z}^N)|p(\theta))}{\lambda}$}.
\end{align}
When the log loss and $\lambda=N$ is used, this minimization is closely related to variational inference (VI) in Bayesian inference. See \citep{germain2016pac} for details.

Under additional moderate assumptions, using $\hat{q}(\theta|\bold{z}^N)$ for the test error, we can derive the following {\bf excess risk (ER)} bound from the PAC-Bayesian generalization bound \citep{alquier2021user}:
\begin{align}\label{base_PAC}
\scalebox{0.95}{$\displaystyle\mathrm{ER}^l(Y|X,\bold{Z}^N,\theta^*):=R^l(Y|X,\bold{Z}^N)-R^l(Y|X,\theta^*)\leq C_1\frac{\ln N}{N^\alpha}$},
\end{align}
where $R^l(Y|X,\theta^*)=\mathbb{E}_{\nu(Z)}l(Y,f_{\theta^*}(X))$ and $\theta^*=\mathrm{argmin}_\theta \mathbb{E}_{\nu(Z)} l(Y,f_\theta(X))$.
Constant $C_1$ depends only on the problem. Since we aim to analyze EU, we do not further discuss the details of the PAC-Bayesian bound. See  Appendix~\ref{app_prelimi_pac} for the explicit conditions of this bound.
 
Although PAC-Bayesian theory focuses on the average test error over posterior distribution, we use predictive distribution for predictions in  Bayesian inference. Thus we define {\bf Prediction Risk (PR)}:
 \begin{align}\label{eq_pr}
     \mathrm{PR}^l(Y|X,\bold{Z}^N):=\mathbb{E}_{\nu(\bold{Z}^N)} \mathbb{E}_{\nu(Z)} l(Y, \mathbb{E}_{q(\theta|\bold{Z}^N)}f_\theta(X)).
 \end{align}
 When the loss is log loss,  $\mathrm{PR}^{\log}(Y|X,\bold{Z}^N)=-\mathbb{E}_{\nu(\bold{Z}^N)} \mathbb{E}_{\nu(Z)}\log p^q(Y|X,\bold{Z}^N)$ where $p^q(y|x,\bold{z}^N):=\mathbb{E}_{q(\theta|\bold{z}^N)}p(y|x,\theta)$ is the approximate predictive distribution. Thus, $\mathrm{PR}^{\log}(Y|X,\bold{Z}^N)$ corresponds to the log loss of the predictive distribution, which is commonly used in the analysis of Bayesian inference \citep{watanabe2009algebraic,watanabe2018mathematical}.

\subsection{Epistemic uncertainty measurements}\label{sec_prac}
Here, we introduce widely used EU measurements in approximate Bayesian inference. For the log loss, conditioned on $(X,\bold{Z}^N)=(x,\bold{z}^N)$, the approximate mutual information has been widely used for uncertainty estimation \citep{depeweg2018decomposition}:
\begin{align}\label{freq_mutual_information}
    I_{\nu}(\theta;Y|x,\bold{z}^N)
=H[p^q(Y|x,\bold{z}^N)]-\mathbb{E}_{q(\theta|\bold{z}^N)}H[p(Y|x,\theta)],
\end{align}
where $H[p^q(Y|x,\bold{z}^N)]:=-\mathbb{E}_{p^q(Y|x,\bold{z}^N)}\log p^q(Y|x,\bold{z}^N)$ is the entropy of the approximate predictive distribution and $\mathbb{E}_{q(\theta|\bold{z}^N)}H[p(Y|x,\theta)]$ is the conditional entropy. $I_{\nu}(\theta;Y|x,\bold{z}^N)$ has been used in Bayesian experimental design \citep{foster2019variational} and reinforcement learning \citep{depeweg2018decomposition}. Note that  by taking the expectation, we have $I_{\nu}(\theta;Y|X,\bold{Z}^N)=\mathbb{E}_{\nu(X=x)}I_{\nu(\bold{Z}^N=\bold{z}^N)}(\theta;Y|x,\bold{z}^N)$.

In the case of squared loss, the variance of the model is often used for EU. This is a common practice in VI, Monte Carlo (MC) dropout \citep{kendall2017uncertainties}, and deep ensemble methods \citep{lakshminarayanan2017simple}. Conditioned on $(X,\bold{Z}^N)=(x,\bold{z}^N)$, it is written as
\begin{align}\label{drop_out}
    \mathrm{Var}_{\theta|\bold{z}^N}f_\theta(x)=\mathbb{E}_{q(\theta|\bold{z}^N)}(f_\theta(x)-\mathbb{E}_{q(\theta|\bold{z}^N)}f_\theta(x))^2.
\end{align}
Although Eqs.\eqref{freq_mutual_information} and \eqref{drop_out} are widely used in application, there have been limited theoretical study for them as discussed in Sec.~\ref{ch:intro}.

\subsection{Excess risk analysis and epistemic uncertainty}\label{sec:Bayes_SBackground}
Recently, the analysis of EU based on excess risk was proposed \citep{xu2020minimum}. 
The key idea of this analysis is to assume that our statistical model $p(y|x,\theta)$ is correct and address the average performance of this model by assuming a prior distribution over $\theta$ with distribution $p(\theta)$. Specifically, the joint distribution of the training data, the test data, and parameter of the model is given as  $p_B(\bold{Z}^N,Z,\theta):=p(\theta) \prod_{n=1}^Np(Y_n|X_n,\theta)\nu(X_n)p(Y|X,\theta)\nu(X)$. Under this setting, the goal of learning is to infer decision rule $\psi:\mathcal{Z}^N\times \mathcal{X}\to \mathcal{A}$ that minimizes expected loss $\mathbb{E}_{p_B(\bold{Z}^N,Z,\theta)}[l(Y,\psi(X,\bold{Z}^N))]$. They refer to this setting as Bayesian learning since we marginalize out parameter $\theta$. With this notation, \cite{xu2020minimum} defined {\bf minimum excess risk} as 
\begin{align}\label{MER}
    \mathrm{MER}^l(Y|X,\bold{Z}^N):=\inf_{\psi:\mathcal{Z}^N\times \mathcal{X}\to \mathcal{A}}&\mathbb{E}_{p_B(\bold{Z}^N,Z,\theta)}[l(Y,\psi(X,\bold{Z}^N))]\nonumber \\
    &- \inf_{\phi:\Theta\times\mathcal{X}\to\mathcal{A}}\mathbb{E}_{p_B(\bold{Z}^N,Z,\theta)}[l(Y,\phi(\theta,X))],
\end{align}
where the first term is the minimum achievable risk using the training data and the second term is the Bayes risk since it uses  learning rule $\phi:\Theta\times\mathcal{X}\to\mathcal{A}$, which takes true parameter $\theta$ instead of the training data. Thus, the second term is the aleatroic uncertainty (AU) since it expresses the task's fundamental difficulty.
Then $\mathrm{MER}$ can be regarded as the EU since it is the difference between the total risk and the AU \citep{xu2020minimum,hafez2021rate}. 

For the log loss, the first term is $H[p(Y|X,\bold{Z}^N)]$ and the second term is $\mathbb{E}_{p(\theta)} H[p(Y|X,\theta)]$. Thus, $\mathrm{MER}^{\log}(Y|X,\bold{Z}^N)=I(\theta;Y|X,\bold{Z}^N)$, which is the conditional mutual information.
Other than the log loss, if the loss function satisfies the $\sigma^2$ sub-Gaussian property conditioned on $(X,\bold{Z}^N)=(x,\bold{z}^N)$, $\mathrm{MER}^l(Y|X,\bold{Z}^N)\leq\sqrt{2\sigma^2 I(\theta;Y|X,\bold{Z}^N)}$ holds \citep{xu2020minimum}. In many practical settings $I(\theta;Y|X,\bold{Z}^N)$ is upper-bounded by $\mathcal{O}(\ln N/N)$. Thus, EU converges with $\mathcal{O}(\ln N/N)$ under this settings. See Appendix~\ref{app_compare} for more details. 

Although this analysis successfully defined EU with rigorous theoretical analysis, the assumptions are clearly impractical since we assume that the correct model, exact Bayesian posterior, and predictive distributions are available. To extend this analysis into approximate Bayesian inference, it is tempting to combine the theory of MER with PAC-Bayesian theory where the data are generated i.i.d from $\nu(Z)$. For that extension, here we introduce the {\bf Prediction Excess Risk (PER)} using Eq.\eqref{eq_pr}:
\begin{align}\label{eq_erfreq}
&\mathrm{PER}^l(Y|X,\bold{Z}^N):=\mathrm{PR}^l(Y|X,\bold{Z}^N)-\inf_{\tilde{\phi}:\mathcal{X}\to\mathcal{A}}\mathbb{E}_{\nu(Z)}[l(Y,\tilde{\phi}(X))],
\end{align}
where the second term corresponds to the Bayes risk. Although we introduced this definition inspired by Eq.\eqref{MER}, it is impractical for evaluating EU. In practice, we are interested in evaluating EU using only input $x$, as shown in Eqs.\eqref{freq_mutual_information} and \eqref{drop_out}. However, we cannot use Eq.\eqref{eq_erfreq} for that purpose since we do not know both $\nu(Y|x)$ and the Bayes risk in the second term. Despite less practical definition, as shown in Sec.\ref{sec_freq}, PER plays a fundamental role in understanding the algorithm-dependent behavior of the widely used EU measurements in Eqs.\eqref{freq_mutual_information} and \eqref{drop_out}.

\section{Analysis of epistemic uncertainty based on excess risk}\label{sec_freq}
In this section, we develop theories for analyzing the widely used EU measurements introduced in Sec~\ref{sec_prac}. We focus on the following questions. {\bf (Q1)} The convergence behaviors of those measurements are not apparent. As the number of training data points increases, we expect these measurements to converge to zero. {\bf (Q2)} The relationship between these measurements and the generalization error is unclear. Since these measurements depend on the training data and the algorithm, we expect some meaningful relationships must exist. All the proofs in this section are shown in Appendix~\ref{app_proof_sec_3}.

\subsection{Relation between epistemic uncertainty and Bayesian excess risk}\label{sec_freq_bayes}
First, to connect the practical EU evaluation methods in Sec.\ref{sec_prac} with the excess risk analysis in Sec.\ref{sec:Bayes_SBackground},  we introduce the approximate joint distribution of test data, training data, and parameters:
\begin{align}\label{generating2}
    \nu(\bold{Z}^N) q({\theta|\bold{Z}^N}) \nu(Z)\approx p^q(\theta,\bold{Z}^N, Z):=\nu(\bold{Z}^N) q({\theta|\bold{Z}^N}) \nu(X)p(Y|X,\theta).
\end{align}
When the log loss is used, we employ model for $p(y|x,\theta)$ in Eq.\eqref{generating2}. When the squared loss is used, we assume Gaussian distribution $p(y|x,\theta)=N(y|f_\theta(x),v^2)$ for some $v\in\mathbb{R}$.

When a model is well specified, that is, $\nu(y|x)=p(y|x,\theta^*)$ holds for some $\theta^*\in\Theta$, we expect that the predictive distribution converges to $p(y|x,\theta^*)$ and the approximation of Eq.\eqref{generating2} becomes accurate as $N$ increases. We discuss the quality of this approximation in Sec.\ref{sec_pactob}. Under this setting, we define {\bf Bayesian Excess risk (BER)}:
\begin{align}\label{excss_bayes}
    \mathrm{BER}^l(Y|X,\bold{Z}^N):=\mathrm{BPR}^l(Y|X,\bold{Z}^N)-\inf_{\phi:\Theta\times \mathcal{X}\to\mathcal{A}}\mathbb{E}_{p^q(\theta,\bold{Z}^N,Z)}l(Y,\phi(\theta,X))
\end{align}
where $\mathrm{BPR}$ is the {\bf Bayesian Prediction Risk} defined as
\begin{align}
    \mathrm{BPR}^l(Y|X,\bold{Z}^N):=\mathbb{E}_{p^q(\theta,\bold{Z}^N,Z)}l(Y,\mathbb{E}_{q(\theta'|\bold{Z}^N)}f_{\theta'} (X)),
\end{align}
and the second term is the Bayes risk under the approximate joint distribution of Eq.\eqref{generating2}. Note that $\mathrm{BER}^l(Y|X,\bold{Z}^N)$ is always larger than 0; see Appendix~\ref{eq_BER0} for details. We also show the formal definitions of BER and BPR conditioned on $(X,\bold{Z}^N)=(x,\bold{z}^N)$ in Appendix~\ref{app_conditional_relations}.

$\mathrm{BER}$ and $\mathrm{BPR}$ are defined, motivated by PER, PR, and MER. The difference is the mechanism of through which the test data are generated. In $\mathrm{BER}$, we assume that our model $p(y|x,\theta)$ is correct, and the parameters follow the approximate posterior distribution $q({\theta|\bold{z}^N})$. Thus, the data-generating mechanism resembles the setting in Sec.\ref{sec:Bayes_SBackground}. Therefore, similar to MER, BER implies the loss due to insufficient data under the assumption that our current model $q({\theta|\bold{z}^N}) p(y|x,\theta)$ is correct. 

The next theorem elaborate this intuition and connects BER to widely used EU measurements:
\begin{thm}\label{freq_to_bayes}
Conditioned on $(x,\bold{z}^N)$, we express $\mathrm{BER}^{\mathrm{log}}(Y|x,\bold{z}^N)$ for the log loss and $\mathrm{BER}^{(2)}(Y|x,\bold{z}^N)$ for the squared loss. Under the definition of Eq.\eqref{excss_bayes}, we have
\begin{align}
    \mathrm{BER}^{\mathrm{log}}(Y|x,\bold{z}^N)=I_{\nu}(\theta;Y|x,\bold{z}^N), \quad \quad\quad\mathrm{BER}^{(2)}(Y|x,\bold{z}^N)= \mathrm{Var}_{\theta|\bold{z}^N}f_\theta(x).
\end{align}
\end{thm}
Thus, by studying BER, we can analyze the widely used EU measurements. We point out that $\mathrm{BPR}^{\mathrm{log}} (Y|x,\bold{z}^N)=H[p^q(Y|x,\bold{z}^N)]$, and the Bayes risk is given as $\mathbb{E}_{q(\theta|\bold{z}^N)}H[p(Y|x,\theta)]$.

\begin{rmk}\label{rmk_1}
BER captures EU when our model is well specified. On the other hand, PER (Eq.~\eqref{eq_erfreq}) represents the prediction performance, which has the relation to the quality of approximation of Eq.~\eqref{generating2} under the given loss function. This intuition leads to our new VI in Sec.\ref{sec_algorithm}. 
\end{rmk}

\subsection{Analysis of excess risk based on PAC-Bayesian theory}\label{sec_pactob}
Based on the definitions introduced in Sec.~\ref{sec_freq_bayes}, here we develop a novel relation between BER and the generalization. First, we show the results of the squared loss. For simplicity, assume $\mathcal{Y}=\mathbb{R}$. See Appendix~\ref{app_multi_dim} for $\mathcal{Y}=\mathbb{R}^d$.
\begin{thm}\label{thm_square}
Conditioned on $(x,\bold{z}^N)$, assume that a regression function is well specified, that is, $\mathbb{E}_{\nu(Y|x)}[Y|x]=f_{\theta^*}(x)$ holds. Then we have
\begin{align}\label{square_joint}
\mathrm{PER}^{(2)}(Y|x,\bold{z}^N)+\mathrm{BER}^{(2)}(Y|x,\bold{z}^N)=\mathrm{ER}^{(2)}(Y|x,\bold{z}^N,\theta^*)\leq R^{(2)}(Y|x,\bold{z}^N).
\end{align}
Furthermore assume that the PAC-Bayesian bound Eq.\eqref{base_PAC} holds, and then we have
\begin{align}\label{eq_bound_square2}
\mathrm{PER}^{(2)}(Y|X,\bold{Z}^N)+\mathrm{BER}^{(2)}(Y|X,\bold{Z}^N)=\mathrm{ER}^{(2)}(Y|X,\bold{Z}^N,\theta^*)\leq\frac{C_1\ln N}{N^\alpha}.
\end{align}
\end{thm}
\begin{proof}
We use the following relation about the Jensen gap and BER:
\begin{lmm}\label{mv_vi}
For any $(x,y)$ and any posterior distribution conditioned on $\bold{Z}^N=\bold{z}^N$, we have
\begin{align}
|y-\mathbb{E}_{q(\theta|\bold{z}^N)}f_\theta(x)|^2+ \mathrm{Var}_{\theta|\bold{z}^N}f_\theta(x)=\mathbb{E}_{q(\theta|\bold{z}^N)}|y-f_\theta(x)|^2.
\end{align}
\end{lmm}
From this lemma, the theorem follows directly.
\end{proof}
\begin{rmk}\label{rmark_2}
When we use a flexible model, such as a deep neural network for $f_\theta(x)$, assumption $\mathbb{E}_{\nu(Y|x)}[Y|x]=f_{\theta^*}(x)$ holds even when we misspecify noise function $\nu(Y|x)$.
\end{rmk}
From Eq.\eqref{square_joint}, $\mathrm{Var}_{\theta|\bold{z}^N}f_\theta(x)$ clearly is a lower bound of excess risk and test error, consistent with the well-known result that the variance of the predictor often underestimates EU \citep{lakshminarayanan2017simple}. From Eq.\eqref{eq_bound_square2}, $\mathrm{BER}^{(2)}(Y|X,\bold{Z}^N)$ converges to 0 with the same order as the PAC-Bayesian bound. 
Finally, we remark that from Lemma~\ref{mv_vi}, we have
\begin{align}\label{eq_22}
   R^{(2)}(Y|X,\bold{Z}^N)= \mathrm{PR}^{(2)}(Y|X,\bold{Z}^N)+\mathrm{BER}^{(2)}(Y|X,\bold{Z}^N).
\end{align}
This indicates that the test error is decomposed into PR and BER. As pointed out in Remark~\ref{rmk_1}, BER is EU under the approximation of Eq.\eqref{generating2}, and PER represents the quality of that approximation, Eq.\eqref{eq_22} suggests that the test error simultaneously regularizes those BER and PER.

Next we show the log loss result. Our analysis requires additional assumption about model $p(y|x,\theta)$. We define log density ratio $L(y,x,\theta,\theta^*):=-\ln p(y|x,\theta)+\ln p(y|x,\theta^*)$.
\begin{assumption}\label{sub_exp}Conditioned on $(x,\theta,\bold{z}^N)$, there exists convex function $h(\rho)$ for $[0,b)$ such that cumulant function $L(y,x,\theta,\theta^*)$ is upper-bounded by $h(\rho)$,i.e., the following inequality holds:
\begin{align}
    \ln\mathbb{E}_{p(Y|x,\theta)}e^{\rho(L(Y,x,\theta,\theta^*)-\mathbb{E}_{p(Y|x,\theta)}L(Y,x,\theta,\theta^*))} \leq h(\rho).
\end{align}
\end{assumption}
For example, if $h(\rho)=\rho^2\sigma^2(x,\theta)/2$ and $b=\infty$, this assumption resembles the $\sigma^2$ sub-Gaussian property given $(x,\theta,\bold{z}^N)$. When considering Gaussian likelihood $p(y|x,\theta)=N(y|f_\theta(x),v^2)$, we have $\scalebox{0.90}{$h(\rho)=\frac{\rho^2}{2v^2}|f_\theta(x)-f_{\theta*}(x)|^2$}$. Thus, $\sigma^2(x,\theta)$ depends on $x$ and $\theta$, and we refer to this $\sigma^2(x,\theta)$ as a sub-Gaussian property. Other than the Gaussian likelihood, when the log loss is bounded, it satisfies the sub-Gaussian property. In this paper, we focus on this sub-Gaussian setting for Assumption~\ref{sub_exp} to clarify the presentation. We show an example of the logistic regression in Appendix~\ref{app_llogistc}. 
\begin{thm}\label{thm_freq_excess_risk}
When the model is well specified, that is, $\nu(y|x)=p(y|x,\theta^*)$ holds and $\sigma^2(x,\theta)$ sub-Gaussian property is satisfied for $L(y,x,\theta,\theta^*)$, as discussed above. Assume that $\mathbb{E}_{q(\theta|\bold{z}^N)}\sigma^2(x,\theta)<\sigma_p^2<\infty$. Conditioned on $(x,\bold{z}^N)$, we have
\begin{align}\label{general}
      \mathrm{PER}^{\mathrm{log}}(Y|x,\bold{z}^N)+\mathrm{BER}^{\mathrm{log}}(Y|x,\bold{z}^N)\leq \sqrt{2\sigma_p^2\mathrm{ER}^{\mathrm{log}}(Y|x,\bold{z}^N,\theta^*)}.
\end{align}
Moreover, assume PAC-Bayesian bound Eq.\eqref{base_PAC}  and  $\scalebox{0.90}{$\mathbb{E}_{\nu(\bold{z}^N)q(\theta|\bold{Z}^N)\nu(X)}\sigma^2(X,\theta)<\sigma_q^2<\infty$}$ hold, and then we have
\begin{align}\label{general2}
    \mathrm{PER}^{\mathrm{log}}(Y|X,\bold{Z}^N)+\mathrm{BER}^{\mathrm{log}}(Y|X,\bold{Z}^N)&\leq \sqrt{2\sigma_q^2\mathrm{ER}^{\mathrm{log}}(Y|X,\bold{Z}^N,\theta^*)}\nonumber \\
    &\leq \sqrt{\frac{2\sigma_q^2C_1\ln N}{N^{\alpha}}}.
\end{align}
\end{thm}
\begin{proof}
Using Assumption~\ref{sub_exp}, we apply change-of-measure inequalities \citep{ohnishi2021novel} to control the approximation error of Eq.\eqref{generating2}. 
\end{proof}
\begin{rmk}
When we use a generalized linear model, the assumption of a well-specified model is relaxed so that $\mathbb{E}_{\nu(Y|x)}[Y|x]$ is well specified, similar to Remark~\ref{rmark_2}. See Appendix~\ref{relax_assump} for details.
\end{rmk}

From Eq.\eqref{general}, $I_{\nu}(\theta;Y|x,\bold{z}^N)$ is a lower bound of the excess risk and test error. From Eq.\eqref{general2},  $\scalebox{0.95}{$I_{\nu}(\theta;Y|X,\bold{Z}^N)$}$ converges in the order of $\scalebox{0.95}{$\mathcal{O}(\sqrt{\ln N/N^{\alpha}})$}$ if we can upper-bound $\scalebox{0.95}{$\mathbb{E}_{\nu(\bold{Z}^N)q({\theta|\bold{Z}^N})\nu(X)}\sigma^2(X,\theta)<\sigma_q^2<\infty$}$. For the Gaussian likelihood, we have $\scalebox{0.95}{$\mathbb{E}_{\nu(\bold{Z}^N)q(\theta|\bold{Z}^N)\nu(X)}\sigma^2(X,\theta)\leq\!2\mathrm{ER}^{\mathrm{log}}(Y|X,\bold{Z}^N,\theta^*)\!\leq\frac{2C_1\ln N}{N^{\alpha}}:=\sigma_q^2$}$. See Appendix~\ref{app_gaussian} for a detail.

In a similar way, we can derive the convergence rate of the entropy of the predictive distribution, which shows $H[p^q(Y|X,\bold{Z}^N)]=\scalebox{0.95}{$H[p(Y|X,\theta^*)]+\mathcal{O}(\sqrt{\ln N/N^{\alpha}})$}$. See Appendix~\ref{sec_entropy} for a formal statement. We show similar results for Theorem~\ref{thm_freq_excess_risk} under sub-exponential property in Appendix~\ref{app_sub_exp}.

In summary, we obtained the convergence of widely used EU measurements and the entropy in the approximate Bayesian inference for the first time. They converges faster than excess risks. Moreover we obtained two messages from Theorems~\ref{thm_square} and \ref{thm_freq_excess_risk}. First, the widely used EU measurements are the lower bounds of the test error and excess risks. This is consistent with the experimental fact that these EU measurements often underestimate EU. Second, the sum of PER and BER is upper-bounded by excess risk (test error). Thus, when minimizing the test error, we also simultaneously minimize PER and BER. This interpretation extends the intuition of Remark~\ref{rmk_1} and leads to a new VI in Sec.~\ref{sec_algorithm}.

\subsection{Novel EU regularization method for variational inference}\label{sec_algorithm}
As seen in Sec.~\ref{sec_pactob}, minimizing the test error leads to minimizing BER and PER. In this section, we discuss this relation using the objective function of VI. As an explicit example, consider a regression problem using $N(y|f_\theta(x),v^2)$. Then from Lemma~\ref{mv_vi}, the loss function of the standard VI (eliminating $\mathrm{KL}(q({\theta|\bold{z}^N})|p(\theta))$) can be written:
\begin{align}\label{gp}
&-\mathbb{E}_{\nu(Z)}\mathbb{E}_{q(\theta|\bold{z}^N)}\ln N(y|f_\theta(x),v^2)\nonumber \\
&=\mathbb{E}_{\nu(Z)}\frac{|Y-\mathbb{E}_{q(\theta|\bold{z}^N)}f_\theta(X)|^2+ \mathrm{Var}_{\theta|\bold{z}^N}f_\theta(X)}{2v^2}+\frac{\ln2\pi v^2}{2} \nonumber \\
&=\frac{\mathrm{PR}^{(2)}(Y|X,\bold{z}^N)+\mathrm{BER}^{(2)}(Y|X,\bold{z}^N)}{2v^2}+\frac{\ln2\pi v^2}{2}.
\end{align}
Eq.\eqref{gp} implies that the standard VI tries to fit the mean of the predictive distribution to target variable $y$ with the regularization term about the variance of the predictor. These terms corresponds to $\mathrm{PR}^{(2)}(Y|X,\bold{z}^N)$ and $\mathrm{BER}^{(2)}(Y|X,\bold{z}^N)$. Note that there is a relation $\mathrm{BER}^{\mathrm{log}}(Y|x,\bold{z}^N)\leq \mathrm{Var}_{\theta|\bold{z}^N}f_\theta(x)/v^2$ for the Gaussian likelihood. This interpretation is consistent with Remark~\ref{rmk_1}.

It has been numerically reported that the standard VI often underestimates EU. Alternative objective functions have been proposed to address this issue. For example, the entropic loss defined as $\mathrm{Ent}_{\alpha}^l(y,x):=-\frac{1}{\alpha}\ln \mathbb{E}_{q(\theta|\bold{z}^N)}e^{-\alpha l(y,f_\theta(x))}$ for $\alpha>0$, which is used in the $\alpha$-divergence dropout ($\alpha$-DO) \citep{li2017dropout} and the second order PAC-Bayesian methods ($2^\mathrm{nd}$-PAC) \citep{NEURIPS2020_3ac48664,futami2021loss}, can capture EU better than the standard VI.  Note that when $\alpha=1$ and the log loss is used, the entropic risk corresponds to the log loss using the predictive distribution. For the Gaussian likelihood, we can upper-bound the entropic risk:
\begin{align}\label{eq_mutual}
&\mathbb{E}_{\nu(Z)}\mathrm{Ent}_{\alpha=1}^{\mathrm{log}}(Y,X)\nonumber \\
&\leq\!
\mathbb{E}_{\nu(Z)}\frac{|Y\!-\!\mathbb{E}_{q(\theta|\bold{z}^N)}f_\theta(X)|^2}{v^2}\!+\!\frac{\mathrm{Var}_{\theta|\bold{z}^N}f_\theta(X)}{v^2}\!-\!\mathrm{BER}^{\mathrm{log}}(Y|X,\bold{z}^N)\!+\!\ln2\pi v^2,
\end{align}
where we used Eq.\eqref{general}. See Appendix~\ref{app_general} for the derivation. Compared to Eq.\eqref{gp}, the entropic risk implicitly introduces a smaller regularization term for BER. This explains why $\alpha$-DO and $2^\mathrm{nd}$-PAC showed larger EU than the standard VI. We show a similar result for the entropic risk of the general log loss other than the Gaussian likelihood in Appendix~\ref{app_general}.

From these relations, balancing BER and PR appropriately leads to a solution that better evaluates the EU. Motivated by the decomposition in Eqs.\eqref{gp} and \eqref{eq_mutual}, we directly control the prediction performance and the Bayesian excess risk for the Gaussian likelihood:
\begin{align}\label{proposal}
    \scalebox{0.93}{$\displaystyle\mathrm{rBER}(\lambda)=\frac{1}{N}\sum_{i=1}^N \frac{|y_i\!-\!\mathbb{E}_{q(\theta|\bold{z}^N)}f_\theta(x_i)|^2}{2v^2}\!+\lambda\frac{\mathrm{Var}_{\theta|\bold{z}^N}f_\theta(x_i)}{2v^2}\!+\frac{\ln2\pi v^2}{2}\!+\frac{1}{N}\mathrm{KL}(q({\theta|\bold{z}^N})|p(\theta))$},
\end{align}
where $0<\lambda\leq 1$ is the coefficient of the BER regularizer. $\mathrm{KL}(q({\theta|\bold{z}^N})|p(\theta))$ is a regularization term motivated by PAC-Bayesian theory. We select $\lambda$ by cross-validation and it should be smaller than 1 since $\lambda=1$ corresponds to the standard VI from Eq.\eqref{gp} and the standard VI often underestimates the EU. We call Eq.\eqref{proposal} the regularized Bayesian Excess Risk VI (rBER) and show the PAC-Bayesian generalization guarantee for our rBER in Appendix~\ref{app_propose}. In Sec.~\ref{sec_nume_eval}, we numerically evaluated this objective function.

rBER can also be seen as an extension of the standard VI. In the standard VI, the test loss is lower-bounded by the sum of the PR and BER with equal weights (Eq.\eqref{eq_22}). rBER has the flexible weights between PR and BER. See Appendix~\ref{app_propose} for a detailed comparison.

\section{Relation to existing work}
The existing theoretical analysis of uncertainty focused on the calibration performance and clarified when a model over- and underestimates uncertainty \citep{nixon2019measuring,bai2021don,naeini2015obtaining,guo2017calibration}. Other than calibration, the analysis of Gaussian processes (GP) has been gaining attention since GP's posterior predictive distribution can be expressed analytically \citep{fiedler2021practical,lederer2019posterior}. Some research focused on the distance or geometry between the test and training data points to derive EU \citep{liu2020simple,tian2021geometric}. Other approaches connect the randomness of the posterior distribution to predictions by the delta method \citep{nilsen2022epistemic}. 
Differently, the information-theoretic approach \citep{xu2020minimum} focused on the loss function of the problem and defined the excess risk as the EU. Loss function-based analysis was proposed in the deterministic learning algorithm \citep{jain2021deup}.
Our theory, which can be regarded as an extension of the information-theoretic approach \citep{xu2020minimum} to approximate Bayesian inference, derived the convergence properties of the variance and the entropy of the posterior predictive distributions.

Although the excess risk bound in Eq.\eqref{base_PAC} has been discussed by PAC-Bayesian theory \citep{alquier2021user}, its relation to the EU has not been investigated. The relationship between PAC-Bayesian theory and Bayesian inference has been investigated in terms of marginal likelihood \citep{germain2016pac,pmlr-v139-rothfuss21a}. Our work established new relationships that connect the uncertainty of the Bayesian predictive distribution and the PAC-Bayesian generalization bound. The information-theoretic approach \citep{xu2020minimum} clarified that EU can be expressed by conditional mutual information. This relation was extended to meta-learning \citep{jose2021information}. However, the researchers assumed that correct models and exact posterior distributions are available. Our proposed analysis relaxes these assumptions.

\section{Numerical experiments}\label{sec:numerical}
In this section, we numerically confirm the theoretical findings in Sec.~\ref{sec_freq} and our proposed rBER in Eq.\eqref{proposal}. We show the detailed experimental settings and additional results in Appendix~\ref{app:experiment}.
\subsection{Numerical evaluation of Theorem~\ref{thm_square}}
We numerically confirm the statement of Theorem~\ref{thm_square}. First, we consider toy data experiments where the true model is $y=0.5x^3+\epsilon$, $\epsilon\sim N(0,1)$, $x\sim N(0,1)$. We consider a Bayesian neural network (BNN) for $f_\theta(x)$ as a 4 layer neural network model with ReLU activation. We approximate the posterior distribution of the parameters of the neural network by Bayes by backpropagation (BBP), \citep{hernandez2015probabilistic}, dropout \citep{kendall2017uncertainties}, and deep ensemble \citep{lakshminarayanan2017simple}. We evaluate $\mathrm{PER}^{(2)}(Y|X,\bold{Z})$, $\mathrm{BER}^{(2)}(Y|X,\bold{Z})$ ($:=\mathbb{E}_{\nu(X)}\mathrm{Var}_{\theta|\bold{Z}^N}f_\theta(X)$), and $R^{(2)}(Y|X,\bold{Z})$ (test error). The results are shown in Fig.~\ref{fig_toy}. Our numerical results satisfy Eq.\eqref{eq_bound_square2} in Theorem~\ref{thm_square}, that is, $\mathrm{PER}^{(2)}(Y|X,\bold{Z})$ and $\mathrm{BER}^{(2)}(Y|X,\bold{Z})$ are upper-bounded by $R^{(2)}(Y|X,\bold{Z})$ and  converge to zero as the number of samples increases. We calculated the Spearman Rank Correlation (SRC) among $\mathrm{PER}^{(2)}(Y|X,\bold{Z})$, $\mathrm{BER}^{(2)}(Y|X,\bold{Z})$, and $R^{(2)}(Y|X,\bold{Z})$ and showed at least $0.97$ suggesting high correlation relation between them.
\begin{figure}[tb!]
 \centering
 \includegraphics[width=0.9\linewidth]{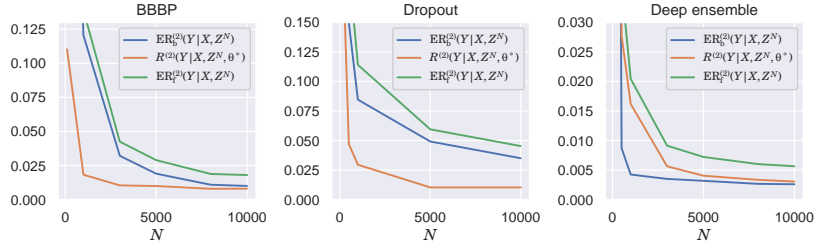}
 \caption{Rresult of toy data experiments: $N$ represents number of training data points, and vertical line is value of each excess risk.}
\label{fig_toy}
\end{figure}
\begin{figure}[tb!]
 \centering
 \includegraphics[width=0.9\linewidth]{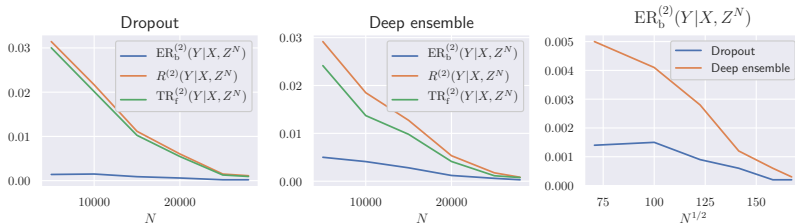}
 \caption{Real data experiments of depth estimation: Vertical line is the value of each risk.}
 \label{fig_ensemble}
\end{figure}

Next, we confirm Theorem~\ref{thm_square} using a real-world dataset. Following the setting of existing work \citep{amini2020deep}, we trained a U-Net style network \citep{ronneberger2015u} with the data of the NYU Depth v2 dataset \citep{silberman2012indoor}, which consists of RGB-to-depth. We applied dropout and deep ensemble methods. Since we cannot evaluate $\mathrm{PER}^{(2)}(Y|X,\bold{Z})$, we instead evaluated Eq.\eqref{eq_22}, which only requires $\mathrm{PR}^{(2)}(Y|X,\bold{Z})$, $\mathrm{BER}^{(2)}(Y|X,\bold{Z})$, and $R^{(2)}(Y|X,\bold{Z})$. The result is shown in Fig.~\ref{fig_ensemble}. We found that $\mathrm{PR}^{(2)}(Y|X,\bold{Z})$, $\mathrm{BER}^{(2)}(Y|X,\bold{Z})$ are upper-bounded by $R^{(2)}(Y|X,\bold{Z})$ for real dataset experiments. We calculated the SRC among $\mathrm{PR}^{(2)}(Y|X,\bold{Z})$, $\mathrm{BER}^{(2)}(Y|X,\bold{Z})$, and $R^{(2)}(Y|X,\bold{Z})$ and showed at least $0.98$, suggesting a high correlation relation between them. We also evaluated the convergence behaviors of BER and show the result on the right in Fig.~\ref{fig_ensemble}. BER converges with $\mathcal{O}(1/N^{1/2})$, which is consistent with Eq.\eqref{eq_bound_square2}.

\subsection{Real data experiments of regularized Bayesian Excess Risk VI}\label{sec_nume_eval}
We numerically compared the prediction and EU evaluation performances of our proposed method shown in Eq.\eqref{proposal} in regression and contextual bandit tasks. Motivated by the success of the entropic risk in particle VI (PVI) \citep{NEURIPS2020_3ac48664,futami2021loss}, which approximates the posterior distribution by the ensemble of models, we also applied our rBER to the PVI setting. Thus, the posterior distribution is expressed as  $q(\theta):=\frac{1}{N}\sum_{i=1}^M\delta_{\theta_i}(\theta)$, where $\delta_{\theta_i}(\theta)$ is the Dirac distribution that has a mass at $\theta_i$. See Appendix~\ref{app:experiment} for details about PVI. 
We refer to rBER($0$) when $\lambda=0$ in Eq.\eqref{proposal}. We compared our method with the existing PVI methods, f-SVGD \citep{DBLP:conf/iclr/WangRZZ19}, $\mathrm{PAC}^2_\mathrm{E}$ \citep{NEURIPS2020_3ac48664}, and VAR \citep{futami2021loss}. 

We used the UCI dataset \citep{Dua:2017} for regression tasks. The model is a single-layer network with ReLU activation, and we used 20 ensembles. The results of $20$ repetitions are shown in Table~\ref{tab:bench_regm}. We evaluated the fitting performance by RMSE and the uncertainty estimation performance by the prediction interval coverage probability (PICP), which shows the number of test observations inside the estimated prediction interval where the interval was set to 0.95. PICP is best when it is close to 0.95. We evaluated the mean prediction interval width (MPIW), which shows an average width of a prediction interval. A smaller MPIW is a better uncertainty estimate when PICP is near the best. Due to space limitations, the results of $\mathrm{PAC}^2_\mathrm{E}$ and the other $\lambda$s and the negative log-likelihood are shown in Appendix~\ref{app:experiment}. We found the existing PVIs show small PICP and MPIW, indicating that the existing methods underestimate the uncertainty. rBER($0$) shows a large PICP and MPIW since the Bayesian excess risk is not regularized. rBER($0.05$) shows a moderate MPIW with a better PICP and almost identical prediction performance in RMSE. Thus, rBER successfully controlled the prediction and the uncertainty evaluation performances.

Next we evaluated the rBER using contextual bandit problems \citep{riquelme2018deep}. We need to balance the trade-off between exploitation and exploration to achieve small cumulative regret. For that purpose, our algorithms must appropriately control the prediction and uncertainty evaluation performance. We used the Thompson sampling algorithm with BNN and two hidden layers. We used 20 ensembles for approximating the posterior distribution. The results of $10$ repetitions are shown in Table~\ref{tab:bench_regretm}. Our approach outperformed other methods, which means our proposed method showed better prediction and uncertainty control than the existing methods.

\begin{figure}
\centering
\begin{minipage}{1.0\textwidth}
\centering
\tblcaption{Benchmark results on test RMSE, PICP, and MPIW.}
\label{tab:bench_regm}
\resizebox{1.\textwidth}{!}{
\begin{tabular}{c|cccc|cccc}
\toprule
\fontsize{9}{7.2}\selectfont \bf{\multirow{2}{*}{Dataset}} & 
\multicolumn{4}{|c}{\fontsize{9}{7.2}\selectfont Avg. Test RMSE} & \multicolumn{4}{|c}{\fontsize{9}{7.2}\selectfont Avg. Test PICP and MPIW in parenthesis} \\
 & \fontsize{9}{7.2}\selectfont f-SVGD&
  \fontsize{9}{7.2}\selectfont VAR&
 \fontsize{9}{7.2}\selectfont rBER($0$) & \fontsize{9}{7.2}\selectfont rBER($0.05$) &
 \fontsize{9}{7.2}\selectfont f-SVGD &
 \fontsize{9}{7.2}\selectfont VAR &
 \fontsize{9}{7.2}\selectfont rBER($0$)  & \fontsize{9}{7.2}\selectfont rBER($0.05$)   \\
\hline
\fontsize{8}{7.2}\selectfont Concrete&
\fontsize{8}{7.2}\selectfont 4.33$\pm$0.8&
\fontsize{8}{7.2}\selectfont 4.30$\pm$0.7 &
\fontsize{8}{7.2}\selectfont 4.47$\pm$0.6 &
\fontsize{8}{7.2}\selectfont 4.48$\pm$0.7 &
\fontsize{8}{7.2}\selectfont 0.82$\pm$0.03 (0.13$\pm$0.00) &
\fontsize{8}{7.2}\selectfont 0.87$\pm$0.04 (0.16$\pm$0.01) &
\fontsize{8}{7.2}\selectfont 0.99$\pm$0.02 (0.50$\pm$0.04)&
\fontsize{8}{7.2}\selectfont {\bf 0.95$\pm$0.02} (0.25$\pm$0.02)  \\
\fontsize{8}{7.2}\selectfont Boston&
\fontsize{8}{7.2}\selectfont 2.54$\pm$0.50 &
\fontsize{8}{7.2}\selectfont 2.53$\pm$0.50 &
\fontsize{8}{7.2}\selectfont 2.53$\pm$0.50 &
\fontsize{8}{7.2}\selectfont 2.53$\pm$0.51 &
\fontsize{8}{7.2}\selectfont 0.63$\pm$0.07 (0.10$\pm$0.02) &
\fontsize{8}{7.2}\selectfont 0.76$\pm$0.05 (0.14$\pm$0.01) &
\fontsize{8}{7.2}\selectfont 0.97$\pm$0.01 (0.33$\pm$0.04) &
\fontsize{8}{7.2}\selectfont {\bf 0.92$\pm$0.04}(0.22$\pm$0.02)  \\
\fontsize{8}{7.2}\selectfont Wine&
\fontsize{8}{7.2}\selectfont 0.61$\pm$0.04 &
\fontsize{8}{7.2}\selectfont 0.61$\pm$0.04 &
\fontsize{8}{7.2}\selectfont 0.64$\pm$0.04 &
\fontsize{8}{7.2}\selectfont 0.63$\pm$0.02 &
\fontsize{8}{7.2}\selectfont 0.79$\pm$0.03 (0.32$\pm$0.05) &
\fontsize{8}{7.2}\selectfont 0.85$\pm$0.02 (0.39$\pm$0.06) &
\fontsize{8}{7.2}\selectfont 0.99$\pm$0.00 (1.61$\pm$0.00) &
\fontsize{8}{7.2}\selectfont {\bf 0.95$\pm$0.03} (0.32$\pm$0.15)   \\
\fontsize{8}{7.2}\selectfont Power&
\fontsize{8}{7.2}\selectfont 3.78$\pm$0.14 &
\fontsize{8}{7.2}\selectfont 3.75$\pm$0.13 &
\fontsize{8}{7.2}\selectfont 3.66$\pm$0.15 &
\fontsize{8}{7.2}\selectfont 3.69$\pm$0.12&
\fontsize{8}{7.2}\selectfont 0.43$\pm$0.01 (0.07$\pm$0.00) &
\fontsize{8}{7.2}\selectfont 0.82$\pm$0.01 (0.15$\pm$0.00) &
\fontsize{8}{7.2}\selectfont 0.99$\pm$0.01 (0.81$\pm$0.01) &
\fontsize{8}{7.2}\selectfont 0.96$\pm$0.01 (0.37$\pm$0.01)    \\
\fontsize{8}{7.2}\selectfont Yacht&
\fontsize{8}{7.2}\selectfont 0.64$\pm$0.28 &
\fontsize{8}{7.2}\selectfont 0.60$\pm$0.28 &
\fontsize{8}{7.2}\selectfont 0.75$\pm$0.41 &
\fontsize{8}{7.2}\selectfont 0.78$\pm$0.48 &
\fontsize{8}{7.2}\selectfont 0.92$\pm$0.04 (0.02$\pm$0.01) &
\fontsize{8}{7.2}\selectfont 0.93$\pm$0.04 (0.04$\pm$0.01) &
\fontsize{8}{7.2}\selectfont {\bf 0.96$\pm$0.03} (0.10$\pm$0.01)&
\fontsize{8}{7.2}\selectfont {\bf 0.94$\pm$0.04} (0.08$\pm$0.01)\\
\fontsize{8}{7.2}\selectfont Protein&
\fontsize{8}{7.2}\selectfont 3.98$\pm$0.54 &
\fontsize{8}{7.2}\selectfont 3.92$\pm$0.05 &
\fontsize{8}{7.2}\selectfont 3.83$\pm$0.10 &
\fontsize{8}{7.2}\selectfont 3.85$\pm$0.05 &
\fontsize{8}{7.2}\selectfont 0.53$\pm$0.01 (0.24$\pm$0.01) &
\fontsize{8}{7.2}\selectfont 0.83$\pm$0.00 (0.58$\pm$0.01) &
\fontsize{8}{7.2}\selectfont 1.0 $\pm$0.00 (5.04$\pm$0.01)&
\fontsize{8}{7.2}\selectfont {\bf 0.96$\pm$0.01} (0.86$\pm$0.00)\\
\bottomrule
\end{tabular}}
\end{minipage}\\
\begin{minipage}{1.0\textwidth}
\centering
\tblcaption{Cumulative regret relative to that of the uniform sampling.}
\label{tab:bench_regretm}
\resizebox{1.0\textwidth}{!}{
\begin{tabular}{c|ccccccc}
\toprule
\fontsize{9}{7.2}\selectfont \bf{Dataset}
 & \fontsize{9}{7.2}\selectfont MAP & \fontsize{9}{7.2}\selectfont $\mathrm{PAC}^2_\mathrm{E}$ & \fontsize{9}{7.2}\selectfont f-SVGD & \fontsize{9}{7.2}\selectfont VAR & 
 \fontsize{9}{7.2}\selectfont rBER($0$) &
 \fontsize{9}{7.2}\selectfont rBER($0.01$) &
 \fontsize{9}{7.2}\selectfont rBER($0.05$)\\
\hline
\fontsize{8}{7.2}\selectfont Mushroom&
\fontsize{8}{7.2}\selectfont 0.129$\pm$0.098 &
\fontsize{8}{7.2}\selectfont 0.037$\pm$0.012 &
\fontsize{8}{7.2}\selectfont 0.043$\pm$0.009 &
\fontsize{8}{7.2}\selectfont 0.029$\pm$0.010 &
\fontsize{8}{7.2}\selectfont 0.075$\pm$0.005 &
\fontsize{8}{7.2}\selectfont {\bf0.024$\pm$0.009} &
\fontsize{8}{7.2}\selectfont {\bf0.021$\pm$0.004} \\
\fontsize{8}{7.2}\selectfont Financial&
\fontsize{8}{7.2}\selectfont 0.791$\pm$0.219 &
\fontsize{8}{7.2}\selectfont 0.189$\pm$0.025 &
\fontsize{8}{7.2}\selectfont 0.154$\pm$0.017 &
\fontsize{8}{7.2}\selectfont 0.155$\pm$0.024 &
\fontsize{8}{7.2}\selectfont 0.351$\pm$0.030 &
\fontsize{8}{7.2}\selectfont {\bf0.075$\pm$0.024} &
\fontsize{8}{7.2}\selectfont {\bf0.075$\pm$0.031}\\
\fontsize{8}{7.2}\selectfont Statlog&
\fontsize{8}{7.2}\selectfont 0.675 $\pm$0.287 &
\fontsize{8}{7.2}\selectfont 0.032$\pm$0.003 &
\fontsize{8}{7.2}\selectfont 0.010$\pm$0.000 &
\fontsize{8}{7.2}\selectfont 0.006$\pm$0.000 &
\fontsize{8}{7.2}\selectfont 0.145$\pm$0.223 &
\fontsize{8}{7.2}\selectfont 0.005$\pm$0.001 &
\fontsize{8}{7.2}\selectfont {\bf0.005$\pm$0.000}\\
\fontsize{8}{7.2}\selectfont CoverType&
\fontsize{8}{7.2}\selectfont 0.610$\pm$0.051 &
\fontsize{8}{7.2}\selectfont 0.396$\pm$0.006 &
\fontsize{8}{7.2}\selectfont 0.372$\pm$0.007 &
\fontsize{8}{7.2}\selectfont {\bf0.291$\pm$0.004} &
\fontsize{8}{7.2}\selectfont 0.610$\pm$0.051 &
\fontsize{8}{7.2}\selectfont 0.351$\pm$0.003 &
\fontsize{8}{7.2}\selectfont {\bf0.290$\pm$0.002}\\
\bottomrule
\end{tabular}}
\end{minipage}
\end{figure}

\section{Conclusion}
We theoretically and numerically analyzed the epistemic uncertainty of approximate inference. We clarified the novel relations among excess risk, epistemic uncertainty, and generalization error. We then showed the convergence rate of the widely used uncertainty measures for the first time. Motivated by theoretical analysis, we proposed a novel variational inference (VI) and applied it to the particle VI. In future work, it would be interesting to explore the relation between BER and evidential learning.

\bibliography{ref.bib}

\newpage
\appendix

\section{Notation}\label{summation}

\centerline{\bf Distributions}
\bgroup
\def\arraystretch{1.5}
\begin{tabular}{p{1in}p{3.25in}}
$\displaystyle \nu(z)$ & A data generating distribution\\
$\displaystyle p(y|x,\theta)$ & A model\\
$\displaystyle p(\theta)$ & A prior distribution\\
$\displaystyle q(\theta|\bold{z}^N)$ & A posterior distribution\\
$\displaystyle p^q(y|x,\bold{z}^N)$ & The predictive distribution obtained by the expectation over $q(\theta|\bold{z}^N)$\\
$\displaystyle p^q(\theta,\bold{z}^N, z)$ & The approximate joint distribution defined as $\nu(\bold{z}^N) q({\theta|\bold{z}^N}) \nu(x)p(y|x,\theta)$\\
$\displaystyle p_B(\bold{z}^N,z,\theta)$ & The joint model used in \cite{xu2020minimum} defined as $p(\theta) \prod_{n=1}^Np(y_n|x_n,\theta)\nu(x_n)p(y|x,\theta)\nu(x)$.\\
\end{tabular}
\egroup
\vspace{0.5cm}\\
\centerline{\bf Risk functions}
\bgroup
\def\arraystretch{1.5}
\begin{tabular}{p{1in}p{3.25in}}
$\displaystyle R^l(Y|X,\bold{Z}^N)$ & A test error  defined as $\mathbb{E}_{\nu(\bold{Z}^N)}\mathbb{E}_{q(\theta|\bold{Z}^N)}\mathbb{E}_{\nu(Z)}l(Y,f_\theta(X))$\\
$\displaystyle \mathbb{E}_{\nu(\bold{Z}^N)}r^l(\bold{Z}^N)$ & A training error defined as $\mathbb{E}_{\nu(\bold{Z}^N)}\mathbb{E}_{q(\theta|\bold{Z}^N)}\sum_{n=1}^N l(Y_n,f_\theta(X_n))/N$\\
$\displaystyle     \mathrm{ER}^l(Y|X,\bold{Z}^N,\theta^*)$ & The excess risk defined as $R^l(Y|X,\bold{Z}^N)-R^l(Y|X,\theta^*)$.\\

$\displaystyle  \mathrm{PR}^l(Y|X,\bold{Z}^N)$ & A prediction risk defined as $\mathbb{E}_{\nu(\bold{Z}^N)} \mathbb{E}_{\nu(Z)} l(Y, \mathbb{E}_{q(\theta|\bold{Z}^N)}f_\theta(X))$\\
$\displaystyle \mathrm{PER}^l(Y|X,\bold{Z}^N)$ & The prediction excess risk defined as $\mathrm{PR}^l(Y|X,\bold{Z}^N)-\inf_{\phi:\mathcal{X}\to\mathcal{A}}\mathbb{E}_{\nu(Z)}[l(Y,\phi(X))]$\\
$\displaystyle     \mathrm{BPR}^l(Y|X,\bold{Z}^N)$ & The Bayesian prediction risk defined as $\mathbb{E}_{p^q(\theta,\bold{Z}^N,Z)}l(Y,\mathbb{E}_{q(\theta'|\bold{Z}^N)}f_{\theta'} (X))$\\
$\displaystyle     \mathrm{BER}^l(Y|X,\bold{Z}^N)$ & The defined Baeysian excess risk as $\mathrm{BPR}^l(Y|X,\bold{Z}^N)-\inf_{\phi:\Theta\times\mathcal{X}\to\mathcal{A}}\mathbb{E}_{p^q(\theta,\bold{Z}^N,Z)}l(Y,\phi(\theta,X))$.\\
$\displaystyle \mathrm{MER}^l(Y|X,\bold{Z}^N)$ & A minimum excess risk used in \cite{xu2020minimum}\\
\end{tabular}
\egroup
\vspace{0.25cm}

\newpage
\section{Summary of settings}
Here we summarize the concepts and definitions of joint distributions and risks used in this work. 
\subsection{Bayesian learning (Sec~\ref{sec:Bayes_SBackground}) [used in \citep{xu2020minimum} and \citep{jose2021information}]}
\begin{itemize}
    \item Joint distribution (All the data is conditinally i.i.d.): \begin{align*}
    p_B(\bold{Z}^N,Z,\theta):=p(\theta) \prod_{n=1}^Np(Y_n|X_n,\theta)\nu(X_n)p(Y|X,\theta)\nu(X).
\end{align*}
 
\item Minimum excess risk:
\begin{align}
    \mathrm{MER}^l(Y|X,\bold{Z}^N)&:=R^l(Y|X,\bold{Z}^N)- R^l(Y|X,\theta)\nonumber \\
    &R^l(Y|X,\bold{Z}^N):=\inf_{\psi:\mathcal{Z}^N\times \mathcal{X}\to \mathcal{A}}\mathbb{E}_{p_B(\bold{Z}^N,Z,\theta)}[l(Y,\psi(X,\bold{Z}^N))], \nonumber \\
    &R^l(Y|X,\theta):=\inf_{\phi:\Theta\times\mathcal{X}\to\mathcal{A}}\mathbb{E}_{p_B(\bold{Z}^N,Z,\theta)}[l(Y,\phi(\theta,X)).\nonumber
\end{align}
\end{itemize}
\subsection{The setting used in the PAC-Bayesian theory(Sec~\ref{sec_freqe})}
\begin{itemize}
    \item Joint distribution of data and parameter $\theta$ (All the data is i.i.d.): \begin{align*}
    \nu(\bold{Z}^N)q(\theta|\bold{Z}^N)\nu(Z).
\end{align*}
\item Prediction excess risk:
\begin{align}
\mathrm{PER}^l(Y|X,\bold{Z}^N)&:=\mathrm{PR}^l(Y|X,\bold{Z}^N)-\inf_{\phi:\mathcal{X}\to\mathcal{A}}\mathbb{E}_{\nu(Z)}[l(Y,\phi(X))],\nonumber\\
    &     \mathrm{PR}^l(Y|X,\bold{Z}^N):=\mathbb{E}_{\nu(\bold{Z}^N)} \mathbb{E}_{\nu(Z)} l(Y, \mathbb{E}_{q(\theta|\bold{Z}^N)}f_\theta(X)).\nonumber
\end{align}
\end{itemize}
\subsection{Our setting defined in Sec~\ref{sec_freq_bayes}}
\begin{itemize}
    \item Joint distribution (The training data is i.i.d. The test data follows the model): \begin{align*}
    p^q(\theta,\bold{Z}^N, Z):=\nu(\bold{Z}^N) q({\theta|\bold{Z}^N}) \nu(X)p(Y|X,\theta)
\end{align*}
\item Bayesian excess risk:
\begin{align}
    \mathrm{BER}^l(Y|X,\bold{Z}^N)&:=\mathrm{BPR}^l(Y|X,\bold{Z}^N)-\inf_{\phi:\Theta\times\mathcal{X}\to\mathcal{A}}\mathbb{E}_{p^q(\theta,\bold{Z}^N,Z)}l(Y,\phi(\theta,X)), \nonumber \\
    & \mathrm{BPR}^l(Y|X,\bold{Z}^N):=\mathbb{E}_{p^q(\theta,\bold{Z}^N,Z)}l(Y,\mathbb{E}_{q(\theta'|\bold{Z}^N)}f_{\theta'} (X))\nonumber.
\end{align}
\end{itemize}

\section{Further preliminaries}\label{app_compare}

\subsection{Additional facts about Bayesian learning}
Here we introduce the preliminary results  \citep{xu2020minimum}  about the MER in a Bayesian setting. Besides the log loss, we can upper bound MER by conditional mutual information. First, we can upper bound MER by the plug-in decision rule. Consider an optimal decision rule $\Psi^*:\mathcal{Z}\times \Theta \to \mathcal{A}$ and this satisfies $R^l(Y|X,\theta)=\inf_{\phi:\Theta\times\mathcal{X}\to\mathcal{A}}\mathbb{E}_{p^g(\bold{Z}^N,Z,\theta)}[l(Y,\phi(\theta,X))=\mathbb{E}l(Y,\Psi^*(X,\theta))$. Then we express $\theta'$ is drawn from a posterior distribution $p(\theta|\bold{Z}^N)$. Then we have
\begin{align}
    MER_l(Y|X,\bold{Z}^N)\leq \mathbb{E}l(Y,\Psi^*(X,\theta'))-l(Y,\Psi^*(X,\theta)),
\end{align}
where $\Psi^*(X,\theta')$ is a plug-in decision rule, first we draw $\theta'$ from posterior distribution and substitute it to $\Psi^*$.
Then if the moment generating function of $l(Y,\Psi^*(X,\theta'))$ under $P(Y,\theta'|X,\bold{Z}^N)$ satisfies regularity conditions, we can upper bound MER. For example, a loss function satisfies $\sigma^2$-subGaussian conditioned on $(X,\bold{z}^N)=(x,\bold{z}^N)$ for all $(x,\bold{z}^N)$, then
\begin{align}\label{MER_def2}
    MER_l(Y|X,\bold{Z}^N)=\sqrt{2\sigma^2 I(\theta;Y|X,\bold{Z}^N)}.
\end{align}
Thus, we can treat zero-one loss and some squared loss.

 Thus, CMI plays a central role in the Bayesian excess risk analysis.
Then existing work \citep{xu2020minimum} shows that
\begin{align}\label{eq_CMI_MI}
    I(\theta;Y|X,\bold{Z}^N)\leq \frac{1}{N}I(\theta;\bold{Z}^N).
\end{align}
This is because the mutual information is upper-bounded by $\mathcal{O}(\ln N)$ for many practical settings. Thus, the excess risk converges to 0 as $N\to \infty$.

\subsection{Preliminaries of the PAC-Bayesian theory}\label{app_prelimi_pac}
We briefly introduce the PAC-Bayesian theory. The typical PAC-Bayesian bound provides us the high-probability guarantee about the gap between the test error $\tilde{R}(\theta):=\mathbb{E}_{\nu(Z)}l(Y,f_\theta(X))$ and $\tilde{r}(\theta):=\frac{1}{N}\sum_{n=1}^Nl(Y_n,f_\theta(X_n))$ (Here we do not take the expectation over $\bold{Z}^N$);
\begin{thm}\citep{alquier2016properties}\label{THM:PAC}
Given a data generating distribution $\nu$, for any prior distribution $p(\theta)$ over $\Theta$ independent of $\bold{Z}^N$ and for any $\xi\in(0,1)$ and $c>0$, with probability at least $1-\xi$ over the choice of training data $\nu(\bold{Z}^N)$, for all probability distributions $q(\theta|\bold{z}^N)$ over $\Theta$, we have
\begin{align}
    \mathbb{E}_{q}\tilde{R}(\theta)\leq \mathbb{E}_{q}\tilde{r}(\theta)+\frac{\mathrm{KL}(q|p)+\ln{\xi}^{-1}+\Omega_{p,\nu}(c,N)}{cN},
\end{align}
where $\Omega_{p,\nu}(c,N):=\ln\mathbb{E}_{p(\theta)}\mathbb{E}_{\nu(\bold{Z}^N)}\mathrm{exp}[cN(\tilde{R}(\theta)-\tilde{r}(\theta))]$.
\end{thm}
This constant $\Psi_{p,\nu}$ depends on the property of the loss function and the data generating distribution and prior. For example, when $l(y,f_\theta(x))-\mathbb{E}_\nu l(Y,f_\theta(X))$ satisfies the $\sigma^2$ sub-Gaussian property, and by setting $c=1/\sqrt{N}$, we have
\begin{align}\label{eq_eqref_sub_gauss_pac}
    \mathbb{E}_{q}\tilde{R}(\theta)\leq \mathbb{E}_{q}\tilde{r}(\theta)+\frac{\mathrm{KL}(q|p)+\ln{\xi}^{-1}+\sigma^2/2}{\sqrt{N}}.
\end{align}
On the other hand, we introduced the bound in expectation in the main paper \citep{alquier2021user}. For example, under the similar setting as Eq.\eqref{eq_eqref_sub_gauss_pac}, when we assume the $\sigma^2$ sub-Gaussian property and $c=\lambda$, we have
\begin{align}
    R^l(Y|X,\bold{Z}^N)\leq r(\theta)+\frac{\mathrm{KL}(q(\theta|\bold{Z}^N)|p(\theta))}{\lambda}+\frac{\lambda\sigma^2}{2N},
\end{align}
see \cite{alquier2021user} for the proof and other settings.

Next, we introduce the PAC-Bayesian bound Eq.\eqref{base_PAC}. 
When $r(\theta)$ satisfies the $L$-lipschitz property and setting the prior as $N(\theta|0,\beta^2I_d)$ and $\lambda=1/\sqrt{N}$, we have
\begin{align}
    R^l(Y|X,\bold{Z}^N)-R^l(Y|X,\theta^*)\leq L\beta\sqrt{\frac{d}{N}}+\frac{\sigma^2}{2\sqrt{N}}+\frac{\frac{\|\theta^*\|^2}{2\beta^2}+\frac{d}{2}\log N}{\sqrt{N}}.
\end{align}
See \citep{alquier2021user} for the proof and other settings.

In the main paper, we considered the squared loss. For the squared loss, above Lipschitz bound cannot be used. Here we introduce different PAC-Bayesian bound. In stead of Lipschitz property, we assume that for asny $\theta,\ \theta'\in\Theta$, there exists a  measurable function $M(x)$ such that
\begin{align}
    f_\theta(x)-f_{\theta'}(x)\leq M(x)\|\theta-\theta'\|^2,
\end{align}
and assume $\mathbb{E}_{\nu(X)}M(X)<L<\infty$. From the almost identical derivation of Example 2.2 in \cite{alquier2021user}, we have
\begin{align}
    R^l(Y|X,\bold{Z}^N)-R^l(Y|X,\theta^*)\leq L\beta^2\frac{d}{N}+\frac{\sigma^2}{2\sqrt{N}}+\frac{\frac{\|\theta^*\|^2}{2\beta^2}+\frac{d}{2}\log N}{\sqrt{N}}.
\end{align}

\section{Proofs of theorems in Section~\ref{sec_freq}}\label{app_proof_sec_3}
Here we present the proofs of Section~\ref{sec_freq}.

\subsection{Conditional expectation of excess risks}\label{app_conditional_relations}
In this section, we define the conditional version of the excess risks. Note that PER and BER was defined as
\begin{align}
&\mathrm{PER}^l(Y|X,\bold{Z}^N):= \mathrm{PR}^l(Y|X,\bold{Z}^N)-\inf_{\tilde{\phi}:\mathcal{X}\to\mathcal{A}}\mathbb{E}_{\nu(Z)}[l(Y,\tilde{\phi}(X))],\nonumber\\
&    \mathrm{BER}^l(Y|X,\bold{Z}^N):=\mathrm{BPR}^l(Y|X,\bold{Z}^N)-\inf_{\phi:\Theta\times\mathcal{X}\to\mathcal{A}}\mathbb{E}_{p^q(\theta,\bold{Z}^N,Z)}l(Y,\phi(\theta,X)).
\end{align}
We define the conditional excess risk as
\begin{align}
    &\mathrm{PER}^l(Y|x,\bold{z}^N):= \mathrm{PR}^l(Y|x,\bold{z}^N)-\inf_{\tilde{\phi}:\mathcal{X}\to\mathcal{A}}\mathbb{E}_{\nu(Y|x)}[l(Y,\tilde{\phi}(x))],\label{def_cond_per}\\
&    \mathrm{BER}^l(Y|x,\bold{z}^N):=\mathrm{BPR}^l(Y|x,\bold{z}^N)-\inf_{\phi:\Theta\times\mathcal{X}\to\mathcal{A}}\mathbb{E}_{q(\theta|\bold{z}^N)p(Y|x,\theta)}l(Y,\phi(\theta,x)).\label{def_cond_ber}
\end{align}
It has been proved in Lemma 3.4 in \cite{steinwart2008support} that if the action space is $\mathbb{R}$ following relation holds
\begin{align}
    \mathbb{E}_{\nu(\bold{Z}^N=\bold{z}^N)\nu(X=x)}\mathrm{PER}^l(Y|x,\bold{z}^N)=\mathrm{PER}^l(Y|X,\bold{Z}^N)\label{rel_cond_per},\\
    \mathbb{E}_{\nu(\bold{Z}^N=\bold{z}^N)\nu(X=x)}\mathrm{BER}^l(Y|x,\bold{z}^N)=\mathrm{BER}^l(Y|X,\bold{Z}^N).\label{rel_cond_ber}
\end{align}
Moreover, for the log loss, from Theorem 3 in \cite{brown1973measurable}, above relation holds. Thus, we can naturally connect the conditional and unconditional definitions of PER and BER for the squared loss and log-loss.

In the following, we explicitly calculate how relations Eqs.~\eqref{rel_cond_per} and \eqref{rel_cond_ber} holds.
The first terms in \eqref{def_cond_per} and \eqref{def_cond_ber} can easily be expressed as
\begin{align}
     \mathrm{PR}^l(Y|X,\bold{Z}^N)&=\mathbb{E}_{\nu(\bold{Z}^N)} \mathbb{E}_{\nu(Z)} l(Y, \mathbb{E}_{q({\theta}|\bold{Z}^N)}f_{\theta}(X))\nonumber \\
     &=\mathbb{E}_{\nu(\bold{Z}^N=\bold{z}^N)\nu(X=x)} \mathbb{E}_{\nu(Y|X=x)} l(Y, \mathbb{E}_{q(\theta|\bold{z}^N)}f_\theta(x))\nonumber \\
     &=\mathbb{E}_{\nu(\bold{Z}^N=\bold{z}^N)\nu(X=x)}\mathrm{PR}^l(Y|x,\bold{z}^N),
\end{align}
and
\begin{align}
     \mathrm{BPR}^l(Y|X,\bold{Z}^N)&=\mathbb{E}_{\nu(\bold{Z}^N)} \mathbb{E}_{\nu(X)q(\theta|\bold{Z}^N)p(Y|X,\theta)} l(Y, \mathbb{E}_{q(\theta'|\bold{Z}^N)}f_{\theta'}(X))\nonumber \\
     &=\mathbb{E}_{\nu(\bold{Z}^N=\bold{z}^N)\nu(X=x)} \mathbb{E}_{q(\theta|\bold{Z}^N=\bold{z}^N)p(Y|X=x,\theta)} l(Y, \mathbb{E}_{q(\theta|\bold{z}^N)}f_\theta(x))\nonumber \\
     &=\mathbb{E}_{\nu(\bold{Z}^N=\bold{z}^N)\nu(X=x)}\mathrm{BPR}^l(Y|x,\bold{z}^N).
\end{align}
Next we calculate the Bayes risks in the second terms. First, for the squared loss, we have
\begin{align}
    \mathbb{E}_{\nu(Z)}[l(Y,\tilde{\phi}(X))]&=\mathbb{E}_{\nu(Z)}(Y-\tilde{\phi}(X))^2\nonumber \\
    &=\mathbb{E}_{\nu(Z)}(Y-\mathbb{E}_{\nu(Y'|X)}Y'|X)^2+\mathbb{E}_{\nu(X)}(\mathbb{E}_{\nu(Y'|X)}Y'|X-\tilde{\phi}(X))^2,
\end{align}
where $\mathbb{E}_{\nu(Y'|X)}Y'|X$ is the conditional expectation. Thus, infimum is achieved by setting $\mathbb{E}_{\nu(Y'|X)}Y'|X=\tilde{\phi}(X)$. Thus, we have
\begin{align}
    \inf_{\tilde{\phi}:\mathcal{X}\to\mathcal{A}}\mathbb{E}_{\nu(Z)}[l(Y,\tilde{\phi}(X))]&=\mathbb{E}_{\nu(Z)}(Y-\mathbb{E}_{\nu(Y'|X)}Y'|X)^2\nonumber \\
    &=\mathbb{E}_{\nu(X=x)}\mathbb{E}_{\nu(Y|X=x)}(Y-\mathbb{E}_{\nu(Y'|x)}Y'|x)^2\nonumber \\
    &=\mathbb{E}_{\nu(X=x)}\inf_{\tilde{\phi}:\mathcal{X}\to\mathcal{A}}\mathbb{E}_{\nu(Y|x)}[l(Y,\tilde{\phi}(x))].
\end{align}
We can show the same statement as
\begin{align}
    &\inf_{\phi:\Theta\times\mathcal{X}\to\mathcal{A}}\mathbb{E}_{p^q(\theta,\bold{Z}^N,Z)}l(Y,\phi(\theta,X))\nonumber \\
    &=\mathbb{E}_{\nu(\bold{Z}^N=\bold{z}^N)\nu(X=x)}\inf_{\phi:\Theta\times\mathcal{X}\to\mathcal{A}}\mathbb{E}_{q(\theta|\bold{z}^N)p(Y|x,\theta)}l(Y,\phi(\theta,x)).
\end{align}
Combined these relations, we get Eqs.~\eqref{rel_cond_per}.

Next, we discuss the conditional Bayes risk for the log-loss. We can proceed the calculation in the same way as the squared loss. Then we have
\begin{align}
    &\inf_{\tilde{\phi}:\mathcal{X}\to\mathcal{A}}\mathbb{E}_{\nu(Z)}[l(Y,\tilde{\phi}(X))]\nonumber \\
    &=-\mathbb{E}_{\nu(Y|X)\nu(X)}\log \nu(Y|X)=\mathbb{E}_{\nu(X=x)} \inf_{\tilde{\phi}:\mathcal{X}\to\mathcal{A}}\mathbb{E}_{\nu(Y|x)}[l(Y,\tilde{\phi}(x))].
\end{align}
The Bayes risk in BPR has the similar relation. Thus, it is clear that the relations shown in Eqs.~\eqref{rel_cond_per} and \eqref{rel_cond_ber} hold.

\subsection{Proof of $\mathrm{BER}^l(Y|X,\bold{Z}^N)\geq 0$}\label{eq_BER0}
We remark that for the squared loss and log-loss, BPR can be written as
\begin{align}
    \mathbb{E}_{p^\mathrm{f}(\theta,\bold{Z}^N,Z)}l(Y,\mathbb{E}_{q(\theta'|\bold{Z}^N)}f_{\theta'} (X))=\inf_{\psi:\mathcal{Z}^N\times \mathcal{X}\to \mathcal{A}}\mathbb{E}_{p^q(\bold{Z}^N,Z,\theta)}[l(Y,\psi(X,\bold{Z}^N))].
\end{align}
This expression is similar to the definition of the first term in MER. Then, applying the same technique in Lemma 1 \citep{xu2020minimum}, we can show that $\mathrm{BPR}^l(Y|X,\bold{Z}^N)$ satisfies the data processing inequality. Then given a Markov chain, for example, $(X,\bold{Z}^N)-(X,Z^{N+1})-Y$, then we have $\mathrm{BPR}^l(Y|X,\bold{Z}^N) \geq \mathrm{BPR}^l(Y|X,Z^{N+1})$. Consider a Markov chain $(X,\bold{Z}^N)-(X,\theta)-Y$. Then by the data processing inequality, we have $\mathrm{BPR}^l(Y|X,\bold{Z}^N)\geq \inf_{\phi:\Theta\times\mathcal{X}\to\mathcal{A}}\mathbb{E}_{p^q(\theta,\bold{Z}^N,Z)}l(Y,\phi(\theta,X))$ since the Bayes error uses the parameter of $p^q(\bold{z}^N,\theta,z)$ directly. This concludes the proof.

\subsection{Proof of Lemma~\ref{mv_vi}}
By definition,
\begin{align}  
\mathbb{E}_{q(\theta|\bold{z}^N)}\|y-f_\theta(x)\|^2&=y^2-2y\mathbb{E}_{q(\theta|\bold{z}^N)}f_\theta(x)+\mathbb{E}_{q(\theta|\bold{z}^N)}[f_\theta(x)^2]\nonumber \\
&=(y-\mathbb{E}_{q(\theta|\bold{z}^N)}f_\theta(x))^2+\mathbb{E}_{q(\theta|\bold{z}^N)}[f_\theta(x)^2]-[\mathbb{E}_{q(\theta|\bold{z}^N)}f_\theta(x)]^2\nonumber \\
&=(y-\mathbb{E}_{q(\theta|\bold{z}^N)}f_\theta(x))^2+\mathbb{E}_{q(\theta|\bold{z}^N)}[f_\theta(x)-\mathbb{E}_{q(\theta|\bold{z}^N)}f_\theta(x)]^2\nonumber \\
&=(y-\mathbb{E}_{q(\theta|\bold{z}^N)}f_\theta(x))^2+\mathrm{Var}f_\theta(x).
\end{align}
This concludes the proof.

\subsection{Proof of Theorem~\ref{freq_to_bayes}}
 Here we consider the conditional quantities of them. The formal definitions of conditional fundamental limit of learning and total risk are given in Appendix~\ref{app_conditional_relations}. 

For the log loss, we use the property of the entropy. For any probability distributions $p$ and $q$, the entropy satisfies $H[p]:=-\mathbb{E}_p\ln p=-\inf_{q}\mathbb{E}_p\ln q$ . Then, by definition, it is clear that
\begin{align}
    \mathrm{BPR}^{\mathrm{log}}(Y|x,\bold{z}^N)=-\mathbb{E}_{q(\theta|\bold{z}^N)}\mathbb{E}_{p(Y|x,\theta)}\ln\mathbb{E}_{q(\theta'|\bold{z}^N)}p(Y|x,\theta')=H[p^q(Y|x,\bold{z}^N)],
\end{align}
where $p^q(Y|x,\bold{z}^N)=\mathbb{E}_{q(\theta|\bold{z}^N)}p(Y|x,\theta)$ is the conditional predictive distribution. Also, by definition, the Bayes risk for the log loss is given as
 \begin{align}  
\inf_{\phi:\Theta\times\mathcal{X}\to\mathcal{A}}\mathbb{E}_{q(\theta|\bold{z}^N)p(Y|x,\theta)}l(Y,\phi(\theta,x))
&=-\inf_{p'}\mathbb{E}_{q(\theta|\bold{z}^N)p(Y|x,\theta)}[\ln p'(Y|x,\theta)]\nonumber \\
&=-\mathbb{E}_{q(\theta|\bold{z}^N)}\mathbb{E}_{p(Y|x,\theta)}\left[\ln p(Y|x,\theta)\right]\nonumber \\
&=E_{q(\theta|\bold{z}^N)}H[p(Y|x,\theta)].
\end{align}

Thus,
\begin{align}
    &\mathrm{BER}^{\mathrm{log}}(Y|x,\bold{z}^N)\nonumber \\
    &=\mathrm{BPR}^{\mathrm{log}}(Y|x,\bold{z}^N)-\inf_{\phi:\Theta\times\mathcal{X}\to\mathcal{A}}\mathbb{E}_{q(\theta|\bold{z}^N)p(Y|x,\theta)}l(Y,\phi(\theta,x))=I_{\nu}(\theta;Y|x,\bold{z}^N).
\end{align}

Next, for the squared loss, recall that for any random variable $Y$ with distribution $p$, we have
\begin{align}
\inf_a \mathbb{E}_p|Y-a|^2=\mathbb{E}_p|Y-\mathbb{E}_{p(Y')} Y'|^2=\mathrm{Var}(Y)    .
\end{align}
Using this relation, we have
 \begin{align}  
 &\mathrm{BER}^{(2)}(Y|x,\bold{z}^N)\nonumber \\
&=\mathrm{BPR}^{(2)}(Y|x,\bold{z}^N)-\inf_{\phi:\Theta\times\mathcal{X}\to\mathcal{A}}\mathbb{E}_{q(\theta|\bold{z}^N)p(Y|x,\theta)}l(Y,\phi(\theta,x))\nonumber \\
&=\mathbb{E}_{q(\theta|\bold{z}^N)}\mathbb{E}_{p(Y|x,\theta)}\|Y-\mathbb{E}_{q(\theta'|\bold{z}^N)}f_{\theta'}(x)\|^2-\mathbb{E}_{q(\theta|\bold{z}^N)}\mathbb{E}_{p(Y|x,\theta)}\|Y-f_\theta(x)\|^2\nonumber \\
&=-2\mathbb{E}_{q(\theta|\bold{z}^N)}f_\theta(x)\mathbb{E}_{q(\theta'|\bold{z}^N)}f_{\theta'}(x)+(\mathbb{E}_{q(\theta'|\bold{z}^N)}f_{\theta'}(x))^2+2\mathbb{E}_{q(\theta|\bold{z}^N)}(f_\theta(x))^2-\mathbb{E}_{q(\theta|\bold{z}^N)}(f_\theta(x))^2\nonumber \\
&=\mathbb{E}_{q(\theta|\bold{z}^N)}(f_\theta(x))^2-(\mathbb{E}_{q(\theta'|\bold{z}^N)}f_{\theta'}(x))^2\nonumber \\
&=\mathbb{E}_{q(\theta|\bold{z}^N)}\|f_\theta(x)-\mathbb{E}_{q(\theta'|\bold{z}^N)}f_{\theta'}(x)\|^2\nonumber \\
&=\mathrm{Var}_{\theta|\bold{z}^N}f_\theta(x).
\end{align}
This concludes the proof.

\subsection{Proof of Theorem~\ref{thm_square}}
First, we prove Eq.\eqref{square_joint}. This is directly obtained by taking the expectation of $\nu(Y|x)$ for Lemma~\ref{mv_vi}. We can prove Eq.\eqref{square_joint} by the direct calculation. By definition
\begin{align}
    \mathrm{ER}^{(2)}(Y|x,\bold{z}^N,\theta^*)&:=R^{(2)}(Y|x,\bold{z}^N)-R^{(2)}(Y|x,\theta^*)\nonumber \\
    &=\mathbb{E}_{\nu(Y|x)q(\theta|\bold{z}^N)}\|Y-f_\theta(x)\|^2-\mathbb{E}_{\nu(Y|x)}\|Y-\mathbb{E}_{\nu(Y|x)}[Y|x]\|^2\nonumber \\
    &=-2f_{\theta^*}(x)\mathbb{E}_{q(\theta|\bold{z}^N)}f_\theta(x)+\mathbb{E}_{q(\theta|\bold{z}^N)}[f_\theta(x)^2]+f_{\theta^*}(x)^2\nonumber \\
    &=\|f_{\theta^*}(x)-\mathbb{E}_{q(\theta|\bold{z}^N)}f_\theta(x)\|^2+\mathrm{Var}_{\theta|\bold{z}^N}f_\theta(x)\nonumber \\
    &=\mathrm{PER}^{(2)}(Y|x,\bold{z}^N)+\mathrm{BER}^{(2)}(Y|x,\bold{z}^N),
\end{align}
where we used the relation $\mathbb{E}_{\nu(Y|x)}[Y|x]=f_{\theta^*}(x)$.
By definition,
\begin{align}
    \mathrm{ER}^{(2)}(Y|x,\bold{z}^N,\theta^*)=R^{(2)}(Y|x,\bold{z}^N)-R^{(2)}(Y|x,\theta^*)\geq 0,
\end{align}
and for the squared loss for any action $a$, we have $l(y,a)\geq 0$. Combined these, we have
\begin{align}
  \mathrm{ER}^{(2)}(Y|x,\bold{z}^N,\theta^*)\leq R^{(2)}(Y|x,\bold{z}^N). 
\end{align}
This concludes the proof of  Eq.\eqref{square_joint}.

The unconditional relation is derived by using relations in Eqs.~\eqref{rel_cond_per} and \eqref{rel_cond_ber}. Finally, we get Eq.\eqref{eq_bound_square2} by the PAC-Bayesian bound Eq.\eqref{base_PAC}.

\subsection{Discussion about  $\mathcal{Y}=\mathbb{R}^d$}\label{app_multi_dim}

For $\mathcal{Y}=\mathbb{R}^d$, Lemma~\ref{mv_vi} and Theorem~\ref{freq_to_bayes} hold since we can proceed the proof in the same way for $\mathcal{Y}=\mathbb{R}$. Thus, we can proceed the proof of Theorem~\ref{thm_square} for $\mathcal{Y}=\mathbb{R}^d$ in the same way as $\mathcal{Y}=\mathbb{R}$. Thus Theorem~\ref{thm_square} holds in  $\mathcal{Y}=\mathbb{R}^d$.

As for Theorem~\ref{thm_freq_excess_risk}, we consider $p(Y|x,\theta)=N(y|f_\theta(x),\mathrm{diag}(v^2))$, where $v^2\in\mathbb{R}^d$, where $\mathrm{diag}(v^2)$ is the diagonal matrix with each entry is $v_i^2$. Then Theorem~\ref{thm_freq_excess_risk} still holds.

\subsection{Proof of Theorem~\ref{thm_freq_excess_risk}}\label{proof_theorem3}
First, we show Eq.\eqref{general}. Recall the definition $L(y,x,\theta,\theta^*)=\ln p(y|x,\theta^*)-\ln p(y|x,\theta)$. We express this $L(Y,x,\theta)=L(y,x,\theta,\theta^*)$ for simplicity.

We use the following change-of-measure inequality, which is also known as the transportation lemma  \cite{boucheron2013concentration, xu2020minimum, xu2020continuity}.
\begin{lmm}\label{translation_lemma}
Let $W$ be a real-valued integrable random variable with probability distribution $p$. Let $h$ be a convex and continuously differentiable function on a interval $(0,b]$ and assume $h(0)=h'(0)=0$. Define for every $x\geq0$, $h^*(x)=\sup_{0\leq\rho<b}\{\rho x-h(\rho)\}$ and let for every $y\geq 0$, $h^{*-1}(y):=\sup\{x\in\mathbb{R}:h^*(x)\leq y\}$. Then if 
\begin{align}\label{eq_assumption_lemma_trans}
    \ln \mathbb{E}_{p(W)}e^{\rho(W-\mathbb{E}_{p(W)}W)}\leq h(\rho),
\end{align}
is satisfied, for any probability distribution $q$, which is absolutely continuous with respect to $p$ such that $\mathrm{KL}(q|p)\leq \infty$, we have
\begin{align}
    \mathbb{E}_{q(W)} W-\mathbb{E}_{p(W)}W\leq h^{*-1}(\mathrm{KL}(q|p)).
\end{align}
\end{lmm}
The proof of this lemma is shown  in \cite{xu2020minimum} as Theorem 4. Also, this lemma previously appeared in \cite{boucheron2013concentration} as Lemma 4.14.

When $h(\rho)=\rho^2\sigma^2/2$ and $b=\infty$, this assumption is  $\sigma^2$ sub-Gaussian property and $h^{*-1}(x)=\sqrt{2x}$.

Then under the assumption, from the Lemma~\ref{translation_lemma}, given $\theta$ and $x$ and $\bold{z}^N$, we have
\begin{align}
    \mathbb{E}_{p(Y|x,\theta^*)}L(Y,x,\theta)-\mathbb{E}_{p(Y|x,\theta)}L(Y,x,\theta)\leq h^{*-1}(\mathrm{KL}(p(Y|x,\theta^*)|p(Y|x,\theta))).
\end{align}
When we focus on a sub-Gaussian property, we have, 
\begin{align}\label{replacenebt}
    \mathbb{E}_{p(Y|x,\theta^*)}L(Y,x,\theta)-\mathbb{E}_{p(Y|x,\theta)}L(Y,x,\theta)\leq \sqrt{2\sigma^2\mathrm{KL}(p(Y|x,\theta^*)|p(Y|x,\theta)))}.
\end{align}
By taking the expectation $\mathbb{E}_{q(\theta|\bold{z}^N)}$, we have
\begin{align}\label{eq_proof_log_loss}
    &\mathbb{E}_{q(\theta|\bold{z}^N)}\mathbb{E}_{p(Y|x,\theta^*)}L(Y,x,\theta)\nonumber \\
    &\leq \mathbb{E}_{q(\theta|\bold{z}^N)}\mathbb{E}_{p(Y|x,\theta)}L(Y,x,\theta)+\mathbb{E}_{q(\theta|\bold{z}^N)}\sqrt{2\sigma^2\mathrm{KL}(p(Y|x,\theta^*)|p(Y|x,\theta)))}\nonumber \\
    &\leq \mathbb{E}_{q(\theta|\bold{z}^N)}\mathbb{E}_{p(Y|x,\theta)}L(Y,x,\theta)+\nonumber \\
    &\quad\quad\quad\quad\sqrt{\mathbb{E}_{q(\theta|\bold{z}^N)}2\sigma^2}\sqrt{\mathbb{E}_{q(\theta|\bold{z}^N)}\mathrm{KL}(p(Y|x,\theta^*)|p(Y|x,\theta)))},
\end{align}
where we used the Holder inequality in the last line.
From the definition of $L$, 
\begin{align}
    &\mathbb{E}_{q(\theta|\bold{z}^N)}\mathbb{E}_{p(Y|x,\theta^*)}L(Y,x,\theta)\nonumber \\
    &=\mathbb{E}_{q(\theta|\bold{z}^N)}\mathbb{E}_{p(Y|x,\theta^*)}[-\ln p(Y|x,\theta)+\ln p(Y|x,\theta^*)]\nonumber \\
    &\geq\mathbb{E}_{p(Y|x,\theta^*)}[-\ln \mathbb{E}_{q(\theta|\bold{z}^N)}p(Y|x,\theta)+\ln p(Y|x,\theta^*)]\nonumber \\
    &=\mathrm{PER}^{log}(Y|x,\bold{z}^N),
\end{align}
where we used the assumption that $p(y|x,\theta^*)=\nu(y|x)$ and used the Jensen inequality for the logarithmic function.

By definition, we have
\begin{align}\label{proof_eq_result}
    \mathbb{E}_{q(\theta|\bold{z}^N)}\mathrm{KL}(p(Y|x,\theta^*)|p(Y|x,\theta)))=\mathrm{ER}^\mathrm{log}(Y|x,\bold{z}^N,\theta^*).
\end{align}
From the definition of $L$, and since we are considering log loss, we have
\begin{align}\label{proof_eq_result2}
     &-\mathbb{E}_{q(\theta|\bold{z}^N)}\mathbb{E}_{p(Y|x,\theta)}L(Y,x,\theta)\nonumber \\
     &= \mathbb{E}_{q(\theta|\bold{z}^N)}\mathbb{E}_{p(Y|x,\theta)}[-\ln p(Y|x,\theta^*)+\ln p(Y|x,\theta)]\nonumber \\
     &\geq \mathbb{E}_{q(\theta|\bold{z}^N)}\mathbb{E}_{p(Y|x,\theta)}[-\ln \mathbb{E}_{q(\theta|\bold{z}^N)}p(Y|x,\theta)+\ln p(Y|x,\theta)]\nonumber \\
     &=\mathrm{BER}^\mathrm{log}(Y|x,\bold{z}^N).
\end{align}
Combined these, we have
\begin{align}
 &\mathrm{PER}^\mathrm{log}(Y|x,\bold{z}^N)+\mathrm{BER}^\mathrm{log}(Y|x,\bold{z}^N)\leq \sqrt{\mathbb{E}_{q(\theta|\bold{z}^N)}2\sigma^2}\sqrt{\mathrm{ER}^\mathrm{log}(Y|x,\bold{z}^N,\theta^*)}.
\end{align}

Then by applying the assumption of the PAC-Bayesian bound, we can obtain Eq.\eqref{general}

As for Eq.\eqref{general2}, we can proceed the calculation by taking the expectation $\mathbb{E}_{\nu(\bold{Z}^N)q(\theta|\bold{Z}^N)\nu(X=x)}$ instead of $\mathbb{E}_{q(\theta|\bold{z}^N)}$ after Eq.\eqref{replacenebt}. Then we get the result of  Eq.\eqref{general2}.

\subsection{The case of Gaussian likelihood}\label{app_gaussian}
Here we show that we can apply the theorem to the Gaussian likelihood. We define $p(y|x,\theta)=N(y|f_\theta(x),v^2)$. From the definition, we have
\begin{align}
    \mathbb{E}_{\nu(\bold{Z}^N)q(\theta|\bold{Z}^N)\nu(X=x)}\mathbb{E}_{p(Y|x,\theta^*)}L(Y,x,\theta)=\mathrm{PER}^{\mathrm{log}}(Y|X,\bold{Z}^N).
\end{align}
We also have
\begin{align}
    \mathbb{E}_{\nu(\bold{Z}^N)q(\theta|\bold{Z}^N)\nu(X=x)}\mathrm{KL}(p(Y|x,\theta^*)|p(Y|x,\theta)))=\frac{1}{2v^2}\mathbb{E}_{\nu(\bold{Z}^N)q(\theta|\bold{Z}^N)\nu(X=x)}\|f_\theta(x)-f_{\theta^*}(x)\|^2.
\end{align}
Moreover, by directly calculating the definition of the exponential moment, we have
\begin{align}
    &\ln\mathbb{E}_{p(Y|x,\theta)}e^{\rho(L(Y,x,\theta)-\mathbb{E}_{p(Y|x,\theta)}L(Y,x,\theta))}=\frac{\rho^2}{2v^2}\|f_\theta(x)-f_{\theta^*}(x)\|^2.
\end{align}
Then by applying the Lemma~\ref{translation_lemma}, we have
\begin{align}
    \mathbb{E}_{p(Y|x,\theta^*)}L(Y,x,\theta)-\mathbb{E}_{p(Y|x,\theta)}L(Y,x,\theta)&\leq \frac{1}{v^2}\|f_\theta(x)-f_{\theta^*}(x)\|^2.
\end{align}
This implies that
\begin{align}
&\mathbb{E}_{\nu(\bold{Z}^N)q(\theta|\bold{Z}^N)\nu(X)}\sigma^2(X,\theta)\nonumber \\
&\leq\mathbb{E}_{\nu(\bold{Z}^N)q(\theta|\bold{Z}^N)\nu(X)}\frac{1}{v^2}|f_\theta(X)\!-\!f_{\theta*}(X)|^2\!=\!2\mathrm{ER}^{\mathrm{log}}(Y|X,\bold{Z}^N,\theta^*)\!\leq\frac{2C_1\ln N}{N^{\alpha}}:=\sigma_q^2    
\end{align}

Summarizing above, for the Gaussian likelihood, we have
\begin{align}\label{eq_gaussian_trans}
    \mathrm{PER}^{\mathrm{log}}(Y|X,\bold{Z}^N)+\mathrm{BER}^{\mathrm{log}}(Y|X,\bold{Z}^N)\leq 2\mathrm{ER}^\mathrm{log}(Y|X,\bold{Z}^N,\theta^*).
\end{align}

\subsection{Relaxation of assumption in Theorem~\ref{thm_freq_excess_risk}}\label{relax_assump}
We assumed that $\nu(y|x)=p(y|x,\theta^*)$ holds for Theorem~\ref{thm_freq_excess_risk}. We can relax this assumption for specific models. In the proof of Theorem~\ref{thm_freq_excess_risk}, we used assumption $\nu(y|x)=p(y|x,\theta^*)$ for connecting the KL divergence with ER (excess risk) as
\begin{align}
        \mathbb{E}_{q(\theta|\bold{z}^N)}\mathrm{KL}(p(Y|x,\theta^*)|p(Y|x,\theta)))=\mathbb{E}_{q(\theta|\bold{z}^N)}\mathrm{KL}(\nu(Y|x)|p(Y|x,\theta)))=\mathrm{ER}^\mathrm{log}(Y|x,\bold{z}^N,\theta^*),
\end{align}
where the first equality comes from the assumption and the second equality comes from the definition of the excess risk of the log-loss.

It has been proved in Proposition 2 of \cite{heide2020safe} that if the model $p(y|x,\theta)$ is the generalized linear model the assumption can be relaxed. To state that condition, we introduce the definitions of a GLM:
\begin{align}
    p(y|x,\theta):=\mathrm{exp}\left(x^\top\theta y-F(\theta)+r(y)\right).
\end{align}
Here, given $x\in\mathcal{X}\subset \mathbb{R}^d$ and the mean value parameter is given by $g^{-1}(x^\top\theta)$ where $g$ is the link function. $F$ is the normalizing constant and $r$ is the reference measure. With this setting, if the GLM model satisfies
\begin{align}\label{cond_GLM}
    \mathbb{E}_{\nu(Y|x)}[Y|x]=g^{-1}(x^\top\theta).
\end{align}
Then we have
\begin{align}
    \mathrm{KL}(p(Y|x,\theta^*)|p(Y|x,\theta))=\mathrm{KL}(\nu(Y|x)|p(Y|x,\theta)).
\end{align}
This can be proved by the direct calculation. This relation implies that even if $\nu(y|x)\neq p(y|x,\theta^*)$ and Eq.\eqref{cond_GLM} is satisfied, then we have
\begin{align}
        &\mathbb{E}_{q(\theta|\bold{z}^N)}\mathrm{KL}(p(Y|x,\theta^*)|p(Y|x,\theta)))\nonumber \\
        &=\mathbb{E}_{q(\theta|\bold{z}^N)}\mathrm{KL}(\nu(Y|x)|p(Y|x,\theta)))=\mathrm{ER}^\mathrm{log}(Y|x,\bold{z}^N,\theta^*).
\end{align}
Then following the proof of Theorem~\ref{thm_freq_excess_risk} in Appendix~\ref{proof_theorem3}, Theorem~\ref{thm_freq_excess_risk} holds even for $\nu(y|x)\neq p(y|x,\theta^*)$.

The condition Eq.\eqref{cond_GLM} implies that the mean function is well specified. This 

\subsection{Entropy convergence rate}\label{sec_entropy}

\begin{col}\label{entropy}
Under the same assumption as Theorem~\ref{thm_freq_excess_risk}, assume that $\ln p(y|x,\theta)$ satisfies the $\sigma^2$ sub-Gaussian property similary to Assumption~\ref{sub_exp}. Conditioned on $(x,\bold{z}^N)$, we have
\begin{align}
    H[p^q(Y|x,\bold{z}^N)]\leq R^{\mathrm{log}}(Y|x,\bold{z}^N)+2\sqrt{2\sigma_p^2\mathrm{ER}^{\mathrm{log}}(Y|x,\bold{z}^N,\theta^*)},
\end{align}
and if the excess risk bound of the PAC-Bayesian theory  Eq.\eqref{base_PAC} holds, we have
\begin{align}
H[p^q(Y|X,\bold{Z}^N)]=H[p(Y|X,\theta^*)]+\mathcal{O}(\sqrt{\ln N/N^{\alpha}}).
\end{align}
\end{col}
\begin{proof}

If the log loss $-\ln p(Y|x,\theta)$ satisfies the $\sigma^2$ sub-Gaussian property similarly to Assumption~\ref{sub_exp}, conditioned on $(x,\theta,\bold{z}^N)$, we have
\begin{align}
    \mathbb{E}_{p(Y|x,\theta^*)}\ln p(Y|x,\theta)-\mathbb{E}_{p(Y|x,\theta)}\ln p(Y|x,\theta)\leq \sqrt{2\sigma^2\mathrm{KL}(p(Y|x,\theta^*)|p(Y|x,\theta)))}.
\end{align}
Thus, by taking the expectation about $q(\theta|\bold{z}^N)$, we have
\begin{align}
 &\mathbb{E}_{q(\theta|\bold{z}^N)}H[p(Y|x,\theta)]\leq \mathrm{R}^{\log}(Y|x,\bold{z}^N)+\sqrt{\mathbb{E}_{q(\theta|\bold{z}^N)}2\sigma^2}\sqrt{\mathrm{ER}^{\mathrm{log}}(Y|x,\bold{z}^N,\theta^*)} \label{eq_entropystart}.
\end{align}
From Eq.\eqref{general}
\begin{align}
      \mathrm{PER}^{\mathrm{log}}(Y|x,\bold{z}^N)+\mathrm{BER}^{\mathrm{log}}(Y|x,\bold{z}^N)\leq \sqrt{2\sigma_p^2\mathrm{ER}^{\mathrm{log}}(Y|x,\bold{z}^N,\theta^*)},
\end{align}
and note that $\mathrm{PER}^{\mathrm{log}}(Y|x,\bold{z}^N)\geq 0$ and $\mathrm{BER}^{\mathrm{log}}(Y|x,\bold{z}^N)=H[p^q(Y|x,\bold{z}^N)]-\mathbb{E}_{q(\theta|\bold{z}^N)}H[p(Y|x,\theta)]$. Combined these inequalities, we have
\begin{align}
    H[p^q(Y|x,\bold{z}^N)]\leq R^{\mathrm{log}}(Y|x,\bold{z}^N)+2\sqrt{2\mathbb{E}_{q(\theta|\bold{z}^N)}\sigma^2}\sqrt{\mathrm{ER}^{\mathrm{log}}(Y|x,\bold{z}^N,\theta^*)}.
\end{align}
Next, we take the expectation over $\nu(\bold{Z}^N)q(\theta|\bold{Z}^N)\nu(X)$ in Eq.\eqref{eq_entropystart} instead of $q(\theta|\bold{z}^N)$, we have
\begin{align}
    &H[p^q(Y|X,\bold{Z}^N)]-H[p(Y|X,\theta^*)]\nonumber \\
    &\leq R^{\mathrm{log}}(Y|X,\bold{Z}^N)-H[p(Y|X,\theta^*)]+2\sqrt{2\mathbb{E}_{\nu(\bold{Z}^N)q(\theta|\bold{Z}^N)\nu(X)}\sigma^2}\sqrt{\mathrm{ER}^{\mathrm{log}}(Y|X,\bold{Z}^N,\theta^*)}\nonumber \\
    &=\mathrm{ER}^{\mathrm{log}}(Y|X,\bold{Z}^N,\theta^*)+2\sqrt{2\mathbb{E}_{\nu(\bold{Z}^N)q(\theta|\bold{Z}^N)\nu(X)}\sigma^2}\sqrt{\mathrm{ER}^{\mathrm{log}}(Y|X,\bold{Z}^N,\theta^*)}.
\end{align}
Then from the excess risk bound of the PAC-Bayesian theory  Eq.\eqref{base_PAC}, we have $\mathrm{ER}^{\mathrm{log}}(Y|X,\bold{Z}^N,\theta^*)=\mathcal{O}(\ln N/N^{\alpha})$, we get the bound.
\end{proof}

\subsection{Discussion about the logistic regression}\label{app_llogistc}
Here we discuss the relation among Bayesian excess risk, frequentist excess risk, and the generalization ability for the logistic regression. For logistic regression, we define the model as $p(Y=1|x,\theta)=\mathrm{sig}(\theta^\top \phi(x))$ where $\mathrm{sig}=1/(1+e^{-x})$ is the sigmoid function and $\phi(x)$ is the feature vector.

We consider applying the following change-of-measure inequality; let $W$ be a real-valued integrable random variable. If 
\begin{align}\label{eq_assumption_lemma_trans}
    \ln \mathbb{E}_{p(W)}e^{(W-\mathbb{E}W)}<\infty,
\end{align}
is satisfied, then, for any probability distribution $Q$, which is absolutely continuous with respect to $P$ such that $\mathrm{KL}(Q|P)\leq \infty$, we have
\begin{align}\label{eq_transportation}
    \mathbb{E}_{q(W)} W-\mathbb{E}_{p(W)}W\leq \ln \frac{\mathbb{E}_{p(W)}e^{\rho(W-\mathbb{E}_{p(W)}W)} +\mathrm{KL}(q|p)}{\rho}.
\end{align}

Here we assume that $q:=p(Y|x,\theta^*)dy$,  $p:=p(Y|x,\theta)dy$, and $W=L(Y,x,\theta,\theta^*)$ conditioned on $(x,\theta,\bold{z}^N)$. Here $L(y,x,\theta,\theta^*)=-\ln p(y|x,\theta)+\ln p(y|x,\theta^*)$.

Conditioned on $(x,\theta,\bold{z}^N)$,  for $\rho\leq 1$, we have
\begin{align}
    &\ln\mathbb{E}_{p(Y|x,\theta)}e^{\rho(L(Y,x,\theta,\theta^*)-\mathbb{E}_{p(Y|x,\theta)}L(Y,x,\theta,\theta^*))}\nonumber \\
    &=\ln \mathbb{E}_{p(Y|x,\theta)}e^{\rho(-\ln p(Y|x,\theta)+\ln p(Y|x,\theta^*))}+\mathbb{E}_{p(Y|x,\theta)}(\ln p(Y|x,\theta)-\ln p(Y|x,\theta^*))\nonumber \\
    &=  \ln \int p(Y|x,\theta)^{1-\rho} p(Y|x,\theta^*)^\rho dy+\rho\mathrm{KL}(p(Y|x,\theta)|p(Y|x,\theta^*))\nonumber \\
        &=  (\rho-1)D_\rho(p(Y|x,\theta^*)|p(Y|x,\theta))+\rho\mathrm{KL}(p(Y|x,\theta)|p(Y|x,\theta^*)) \nonumber \\
        &\leq \rho\mathrm{KL}(p(Y|x,\theta)|p(Y|x,\theta^*)),
\end{align}
where we used the definition
\begin{align}
    D_\alpha(P|Q):=\frac{1}{\alpha-1}\ln \int \left(\frac{dP}{dQ}\right)^{\alpha-1}dP\geq 0.
\end{align}

Under this definition, we have
\begin{align}
    &\mathrm{KL}(p(Y|x,\theta)|p(Y|x,\theta^*))\nonumber \\
    &\leq \sup_{x\in \mathcal{X},\theta\in \Theta}(-\ln \min\{\mathrm{sig}(\theta^\top \phi(x)),1-\mathrm{sig}(\theta^\top \phi(x))\})\mathrm{TV}(p(Y|x,\theta)|p(Y|x,\theta^*))\nonumber \\
        &\leq \sup_{x\in \mathcal{X},\theta\in \Theta}(-\ln \min\{\mathrm{sig}(\theta^\top \phi(x)),1-\mathrm{sig}(\theta^\top \phi(x))\})\sqrt{\frac{1}{2}\mathrm{KL}(p(Y|x,\theta^*)|p(Y|x,\theta)})\nonumber \\
        &\leq \sup_{x\in \mathcal{X},\theta\in \Theta}(-\ln \min\{\mathrm{sig}(\theta^\top \phi(x)),1-\mathrm{sig}(\theta^\top \phi(x))\})\sqrt{\mathrm{ER}^{\mathrm{log}}(Y|X,\bold{z}^N,\theta^*)},
\end{align}
where we used the Pinsker inequality and used the assumption that the model is well-specified. For simplicity, we  express the coefficient as 
\begin{align}
\rho:=\sup_{x\in \mathcal{X},\theta\in \Theta}(-\ln \min\{\mathrm{sig}(\theta^\top \phi(x)),1-\mathrm{sig}(\theta^\top \phi(x))\}).    
\end{align}
Then, by the transportation lemma using $\rho=1$, we have
\begin{align}
      &\mathrm{PER}^{\mathrm{log}}(Y|x,\bold{Z}^N)+\mathrm{BER}^{\mathrm{log}}(Y|X,\bold{Z}^N)\leq  \rho\sqrt{\mathrm{ER}^{\mathrm{log}}(Y|X,\bold{Z}^N,\theta^*)})+\mathrm{ER}^{\mathrm{log}}(Y|X,\bold{Z}^N,\theta^*).
\end{align}
Then by taking the expectation with respect to $\nu(\bold{Z}^N)\nu(X)$, we have
\begin{align}
      &\mathrm{PER}^{\mathrm{log}}(Y|X,\bold{Z}^N)+\mathrm{BER}^{\mathrm{log}}(Y|X,\bold{Z}^N)\leq  \rho\sqrt{\mathrm{ER}^{\mathrm{log}}(Y|X,\bold{Z}^N,\theta^*)})+\mathrm{ER}^{\mathrm{log}}(Y|X,\bold{Z}^N,\theta^*),
\end{align}
where we used the Jensen inequality. For the logistic model, the PAC-Bayesian bound Eq.\eqref{base_PAC} holds \citep{alquier2021user}, we have
\begin{align}
      &\mathrm{PER}^{\mathrm{log}}(Y|X,\bold{Z}^N)+\mathrm{BER}^{\mathrm{log}}(Y|X,\bold{Z}^N)\leq  \rho\sqrt{C_1\left(\frac{\ln N}{N^\alpha}\right)}+C_1\left(\frac{\ln N}{N^\alpha}\right).
\end{align}
Since $\mathrm{BER}^{\mathrm{log}}(Y|X,\bold{z}^N)=I_\nu(\theta;Y|X,\bold{Z}^N)$, thus the mutual information converges $\mathcal{O}\left(\sqrt{\frac{\ln N}{N^\alpha}}\right)$. Moreover using Collorary~\ref{entropy}, we can derive the convergence of the entropy. Considering the same calculation in the proof of Collorary~\ref{entropy}, we have
\begin{align}
 &\mathbb{E}_{q(\theta|\bold{z}^N)}H[p(Y|x,\theta)]\leq \mathrm{R}^{log}(Y|x,\bold{z}^N) +\rho\sqrt{\mathrm{ER}^{\mathrm{log}}(Y|X,\bold{z}^N,\theta^*)})+\mathrm{ER}^{\mathrm{log}}(Y|x,\bold{z}^N,\theta^*).
\end{align}
Note that $\mathrm{PER}^{\mathrm{log}}(Y|x,\bold{z}^N)\geq 0$ and $\mathrm{BER}^{\mathrm{log}}(Y|x,\bold{z}^N)=H[p^q(Y|x,\bold{z}^N)]-\mathbb{E}_{q(\theta|\bold{z}^N)}H[p(Y|x,\theta)]$, we have
\begin{align}
 &H[p^q(Y|x,\bold{z}^N)]\leq \mathrm{R}^{log}(Y|x,\bold{z}^N)+2\rho\sqrt{\mathrm{ER}^{\mathrm{log}}(Y|X,\bold{z}^N,\theta^*)})+2\mathrm{ER}^{\mathrm{log}}(Y|X,\bold{z}^N,\theta^*).
\end{align}
Thus, the entropy is bounded by the test loss. Then by taking the expectation with respect to $\nu(\bold{Z}^N)\nu(X)$, 
and using the PAC-Bayesian bound Eq.\eqref{base_PAC}. Then we have
\begin{align}
 &H[p(Y|X,\bold{Z}^N)]\leq H[p(Y|X,\theta^*)] +2\rho\sqrt{C_1\left(\frac{\ln N}{N^\alpha}\right)} +3C_1\left(\frac{\ln N}{N^\alpha}\right).
\end{align}
Thus, we have
\begin{align}
    &H[p(Y|X,\bold{Z}^N)]\leq H[p(Y|X,\theta^*)] + \mathcal{O}\left(\sqrt{\frac{\ln N}{N^\alpha}}\right).
\end{align}

\subsection{Discussion of entropic risk}\label{app_general}
Before showing the relation between posterior variance and mutual information, we introduce an important lemma used in the analysis.
\begin{lmm}[Lemma 1 in \citep{haussler1997mutual}]\label{lmm_inside_exp}
Let $P(w)$ be a measure on a set $W$ and $Q(v)$ be a measure on a set $V$. For any real-valued function $u(w,v)$, we have
\begin{align}
    -\int_V dQ(v)\ln \int_W dP(w) e^{u(w,v)}\leq -\ln \int_W dP(w) e^{\int_V dQ(v)u(w,v)}.
\end{align}
\end{lmm}
For completeness, we show the proof.
\begin{proof}
    For any real valued functions $u_1$ and $u_2$ and $0\leq \alpha \leq 1$, we have
    \begin{align}
        \int_W dP(w) e^{\alpha u_1(w)+(1-\alpha) u_2(w)}
        &=\int_W dP(w) \left(e^{u_1(w)}\right)^\alpha \left(e^{u_2(w)}\right)^{1-\alpha}\nonumber \\
        &\leq\left(\int_W dP(w) e^{u_1(w)}\right)^\alpha \left(\int_W dP(w)e^{u_2(w)}\right)^{1-\alpha},
    \end{align}
    where we used H\"{o}lder's inequality. Taking logarithmic function, this shows that $\ln dP(w) e^{u(w,v)}$ is convex in $u$. Thus, the result follows by using the Jensen inequality.
\end{proof}
Thus, this theorem is not restricted to probability distributions.

We first show the relation between posterior variance and the mutual information ($\mathrm{BER}^{\mathrm{log}}(Y|x,\bold{z}^N)$).
\begin{lmm}For the Gaussian likelihood $p(y|x,\theta)=N(y|f_\theta(x),v^2)$, we have
\begin{align}
    \mathrm{BER}^{\mathrm{log}}(Y|x,\bold{z}^N)\leq \frac{\mathrm{Var}f_\theta(x)}{v^2}.
\end{align}
\end{lmm}
\begin{proof}
\begin{align}
\mathrm{BER}^{\mathrm{log}}(Y|x,\bold{z}^N)&=I_{\nu}(\theta;Y|x,\bold{z}^N)\nonumber \\
&=\mathbb{E}_{q(\theta|\bold{z}^N)}\mathbb{E}_{p(Y|x,\theta)}\left[-\ln\mathbb{E}_{q(\theta'|\bold{z}^N)}p(Y|x,\theta')+\ln p(Y|x,\theta)\right]\nonumber \\
&=-\mathbb{E}_{q(\theta|\bold{z}^N)}\mathbb{E}_{p(Y|x,\theta)}\ln\mathbb{E}_{q(\theta'|\bold{z}^N)}e^{\ln p(Y|x,\theta')-\ln p(Y|x,\theta)}.
\end{align}
Then, applying Lemma~\ref{lmm_inside_exp}, we have
\begin{align}
    &\mathrm{BER}^{\mathrm{log}}(Y|x,\bold{z}^N)\nonumber \\
    &\leq-\ln\mathbb{E}_{q(\theta'|\bold{z}^N)}e^{\mathbb{E}_{q(\theta|\bold{z}^N)}\mathbb{E}_{p(Y|x,\theta)}\ln p(Y|x,\theta')-\ln p(Y|x,\theta)}\nonumber \\
    &\leq-\ln\mathbb{E}_{q(\theta'|\bold{z}^N)}e^{-\frac{1}{2v^2}\mathbb{E}_{q(\theta|\bold{z}^N)p(Y|x,\theta)}(Y-f_{\theta'}(x))^2-(Y-f_\theta(x))^2}\nonumber \\
        &\leq-\ln\mathbb{E}_{q(\theta'|\bold{z}^N)}e^{-\frac{1}{2v^2}\mathbb{E}_{q(\theta|\bold{z}^N)}(v^2+f^2_\theta(x)-2f_{\theta'}(x)f_\theta(x)+f_{\theta'}^2(x)-(v^2+f^2_\theta(x)+f^2_\theta(x)-2f_\theta^2(x)))}\nonumber \\
&=-\ln\mathbb{E}_{q(\theta'|\bold{z}^N)}e^{-\frac{1}{2v^2}(\mathbb{E}_{q(\theta|\bold{z}^N)}f_\theta^2(x)-2\mathbb{E}_{q(\theta|\bold{z}^N)}f_\theta(x)f_{\theta'}(x)+f_{\theta'}^{2}(x))}\nonumber \\
&\leq\frac{\mathrm{Var}f_\theta(x)}{v^2}.
\end{align}
\end{proof}

Next we show Eq.\eqref{eq_mutual}. We use the transportation inequality Eq.\eqref{eq_transportation}. Similarly to the derivation of Eq.\eqref{eq_gaussian_trans}, we have
\begin{align}
    &-\mathbb{E}_{q(\theta|\bold{z}^N)}\mathbb{E}_{\nu(Y|x)}\ln\frac{p(Y|x,\theta)}{\nu(Y|x)}+\mathrm{BER}^{\mathrm{log}}(Y|x,\bold{z}^N)\nonumber \\
    &\leq \mathbb{E}_{q(\theta|\bold{z}^N)}\mathrm{KL}(\nu(Y|x)|p(Y|x,\theta))+\frac{1}{2v^2}\mathbb{E}_{q(\theta|\bold{z}^N)}\|f_\theta(x)-f_{\theta^*}(x)\|^2.
\end{align}
Thus, we have
\begin{align}
    &-\mathbb{E}_{\nu(Y|x)}\ln \mathbb{E}_{q(\theta|\bold{z}^N)} N(Y|f_\theta(x),v^2)+\mathbb{E}_{\nu(Y|x)}\ln \nu(Y|x)\nonumber \\
    &\leq \mathbb{E}_{q(\theta|\bold{z}^N)}\frac{1}{2v^2}\|f_\theta(x)-f_{\theta^*}(x)\|^2-\mathbb{E}_{\nu(Y|x)}\mathbb{E}_{q(\theta|\bold{z}^N)}\ln N(Y|f_\theta(x),v^2)\nonumber \\
    &\quad\quad-\mathrm{BER}^{\mathrm{log}}(Y|x,\bold{z}^N)+\mathbb{E}_{\nu(Y|x)}\ln \nu(Y|x)\nonumber \\
    &\leq \mathbb{E}_{\nu(Y|x)q(\theta|\bold{z}^N)}\frac{1}{2v^2}\|Y-f_\theta(x)\|-\mathbb{E}_{\nu(Y|x)}\frac{1}{2v^2}\|Y-f_{\theta^*}(x)\|^2-\mathrm{BER}^{\mathrm{log}}(Y|x,\bold{z}^N)\nonumber \\
    &\quad\quad-\mathbb{E}_{\nu(Y|x)}\mathbb{E}_{q(\theta|\bold{z}^N)}\ln N(Y|f_\theta(x),v^2)+\mathbb{E}_{\nu(Y|x)}\ln \nu(Y|x).
\end{align}
Then, we have
\begin{align}
    &-\mathbb{E}_{\nu(Y|x)}\ln \mathbb{E}_{q(\theta|\bold{z}^N)} N(Y|f_\theta(x),v^2)\nonumber \\
    &\leq \mathbb{E}_{\nu(Y|x)q(\theta|\bold{z}^N)}\frac{1}{2v^2}\|Y-f_\theta(x)\|^2-\mathrm{BER}^{\mathrm{log}}(Y|x,\bold{z}^N)
    -\mathbb{E}_{\nu(Y|x)}\mathbb{E}_{q(\theta|\bold{z}^N)}\ln N(Y|f_\theta(x),v^2)\nonumber \\
    &\leq \mathbb{E}_{\nu(Y|x)q(\theta|\bold{z}^N)}\frac{1}{v^2}\|Y-f_\theta(x)\|^2+\frac{1}{2}\ln 2\pi v^2-\mathrm{BER}^{\mathrm{log}}(Y|x,\bold{z}^N)\nonumber \\
        &\leq \mathbb{E}_{\nu(Y|x)}\frac{1}{v^2}\|Y-\mathbb{E}_{q(\theta|\bold{z}^N)}f_\theta(x)\|^2+\frac{1}{v^2}\mathrm{Var}f_\theta(x)-\mathrm{BER}^{\mathrm{log}}(Y|x,\bold{z}^N)+\frac{1}{2}\ln 2\pi v^2.
\end{align}
This concludes the proof of Eq.\eqref{eq_mutual}.

Next, we discuss the entropic risk for the general log-likelihood other than the Gaussian distribution. 
Note that
\begin{align}
\mathbb{E}_{\nu(Y|x)}\mathrm{Ent}_{\alpha=1}^{\mathrm{log}}(Y,x)=-\mathbb{E}_{\nu(Y|x)}\ln\mathbb{E}_{q(\theta|\bold{z}^N)}p(Y|x,\theta).
\end{align}
Then from Eq.\eqref{general2}, we have
\begin{align}
    \mathrm{PER}^{\mathrm{log}}(Y|x,\bold{z}^N)&\leq \sqrt{2\sigma_p^2\mathrm{ER}^{\mathrm{log}}(Y|x,\bold{z}^N,\theta^*)}-\mathrm{BER}^{\mathrm{log}}(Y|x,\bold{z}^N).
\end{align}
Then by using the Cauchy-Schwartz inequality and since we assumed that the model well-specified, we have
\begin{align}
    -\mathbb{E}_{\nu(Y|x)}\ln\mathbb{E}_{q(\theta|\bold{z}^N)}p(Y|x,\theta)&\leq -\mathbb{E}_\nu\mathbb{E}_{q(\theta|\bold{z}^N)}\ln p(Y|x,\theta)-\mathrm{BER}^{\mathrm{log}}(Y|x,\bold{z}^N)+\frac{\sigma_p^2}{2}.
\end{align}
Thus, the entropic risk has small regularization about BER.

\section{Discussion of sub-exponential and sub-Gamma assumption}\label{app_sub_exp}
In the main paper, we have shown Theorem~\ref{thm_freq_excess_risk} and Collorary~\ref{entropy} when the sub-Gaussian property holds. Here we present the results for sub-exponential and sub-Gamma property.
Conditioned on $\theta$, $x$, and $\bold{z}^N$, by using Lemma~\ref{translation_lemma} we have
\begin{align}
    \mathbb{E}_{p(Y|x,\theta^*)}L(Y,x,\theta)-\mathbb{E}_{p(Y|x,\theta)}L(Y,x,\theta)\leq h^{*-1}(\mathrm{KL}(p(Y|x,\theta^*)|p(Y|x,\theta))).
\end{align}
Then by proceeding the calculation in the same way as the proof of Theorem~\ref{thm_freq_excess_risk}, we have
\begin{align}
        &\mathrm{PER}^{\mathrm{log}}(Y|X,\bold{Z}^N)+\mathrm{BER}^{\mathrm{log}}(Y|X,\bold{Z}^N)\nonumber \\
        &\leq \mathbb{E}_{\nu(\bold{Z}^N)q(\theta|\bold{Z}^N)\nu(X=x)}h^{*-1}(\mathrm{KL}(p(Y|x,\theta^*)|p(Y|x,\theta))),
\end{align}
and since $h^{*-1}$ is concave, we have
\begin{align}
        &\mathrm{PER}^{\mathrm{log}}(Y|X,\bold{Z}^N)+\mathrm{BER}^{\mathrm{log}}(Y|X,\bold{Z}^N)\nonumber \\
        &\leq h^{*-1}(\mathbb{E}_{\nu(\bold{Z}^N)q(\theta|\bold{Z}^N)\nu(X=x)}\mathrm{KL}(p(Y|x,\theta^*)|p(Y|x,\theta))),
\end{align}
We need to derive $h^{*-1}$. For example in Lemma~\ref{translation_lemma}, if $h(\lambda)=\frac{\sigma^2\lambda^2}{2}$ for $0\leq \lambda \leq 1/b$, then we have
\begin{align}
    h^{*-1}(y)=\begin{cases}
    \sqrt{2\sigma^2 y}\quad if\ y\leq \frac{\sigma^2}{2b}\\
    by+\frac{\sigma^2}{2b}\quad otherwise.
    \end{cases}
\end{align}
This is called as sub-exponential property. Thus, if for some constant $b$ and $\sigma^2$, the following relation holds,
\begin{align}
    \mathbb{E}_{\nu(\bold{Z}^N)q(\theta|\bold{Z}^N)\nu(X=x)}\mathrm{KL}(p(Y|x,\theta^*)|p(Y|x,\theta))=\mathrm{ER}^{\log} (Y|X,\bold{Z}^N,\theta^*)\leq \frac{\sigma^2}{2b},
\end{align}
then we have
\begin{align}
        &\mathrm{PER}^{\mathrm{log}}(Y|X,\bold{Z}^N)+\mathrm{BER}^{\mathrm{log}}(Y|X,\bold{Z}^N)\leq \sqrt{2\sigma^2\mathrm{ER}^{\log} (Y|X,\bold{Z}^N,\theta^*)},
\end{align}
and otherwise we have
\begin{align}
        &\mathrm{PER}^{\mathrm{log}}(Y|X,\bold{Z}^N)+\mathrm{BER}^{\mathrm{log}}(Y|X,\bold{Z}^N)\leq b\mathrm{ER}^{\log} (Y|X,\bold{Z}^N,\theta^*)+\frac{\sigma^2}{2b}.
\end{align}

Next, in Lemma~\ref{translation_lemma}, if $h(\lambda)=\frac{\sigma^2\lambda^2}{2(1-c|\lambda|)}$ for $0\leq \lambda \leq 1/c$ and $c>0$, then we have
\begin{align}
    h^{*-1}(y)=\sqrt{2\sigma^2 y}+cy.
\end{align}
This is called as sub-gamma property. If $\sigma^2$ and $c$ are upper bounded by positive constants $\sigma^2<\sigma_0^2<\infty$ and $c<c_0<\infty$, then we have
\begin{align}
        &\mathrm{PER}^{\mathrm{log}}(Y|X,\bold{Z}^N)+\mathrm{BER}^{\mathrm{log}}(Y|X,\bold{Z}^N)\leq \sqrt{2\sigma_0^2\mathrm{ER}^{\log} (Y|X,\bold{Z}^N,\theta^*)}+c_0\mathrm{ER}^{\log} (Y|X,\bold{Z}^N,\theta^*).
\end{align}

\section{Detailed description of the proposed method}\label{app_propose}
First, we show the PAC-Bayesian bound for our proposed method. Following the high-probability bound of Theorem~\ref{THM:PAC}, given a distribution $\nu(Z)$, for any prior distribution $p(\theta)$ over $\Theta$ independent of $\bold{z}^N$ and for any $\xi\in(0,1)$ and $c>0$, with probability at least $1-\xi$ over the choice of training data $\bold{z}^N$, for all probability distributions $q(\theta|\bold{z}^N)$ over $\Theta$, we have
\begin{align}\small\label{bound_ber}
    &\mathbb{E}_{\nu(Z)}\left[\frac{|Y-\mathbb{E}_{q(\theta|\bold{z}^N)}f_\theta(X)|^2}{2v^2}\!+\lambda\frac{\mathrm{Var}_{\theta|\bold{z}^N}f_\theta(X)}{2v^2}\right]+\frac{\ln2\pi v^2}{2}\nonumber \\
    &\leq \frac{1}{N}\sum_{i=1}^N\left[\frac{\mathbb{E}_{q(\theta|\bold{z}^N)}|y_i\!-\!\mathbb{E}_{q(\theta|\bold{z}^N)}f_\theta(x_i)|^2}{2v^2}\!+\!\lambda\frac{\mathrm{Var}_{\theta|\bold{z}^N}f_\theta(x_i)}{2v^2}\right]\!+\!\frac{\ln2\pi v^2}{2}\nonumber \\
    &\quad\quad\quad\quad\!+\!\frac{\mathrm{KL}(q|p)\!+\!\frac{1}{2}\ln\frac{1}{\xi}\!+\frac{1}{2}\Omega_{p,\nu}(c\!,\!N)}{cN},
\end{align}
where 
\begin{align}
    &\Omega_{p,\nu}(c,N):=\ln\mathbb{E}_{p(\theta)p(\theta')}\mathbb{E}_{\nu(\bold{Z})^N}\mathrm{exp}[cN(\mathbb{E}_{\nu(Z)}L(Z,\theta,\theta')-\frac{1}{N}\sum_{n=1}^NL(Z_n,\theta,\theta'))],\\
    &L(z,\theta,\theta'):=\mathbb{E}_{q(\theta|\bold{z}^N)}\frac{|y-f_\theta(x)|^2}{2v^2}\!+(\lambda-1)\frac{\mathbb{E}_{q(\theta|\bold{z}^N)}f_\theta^2(x)-\mathbb{E}_{q(\theta|\bold{z}^N)}\mathbb{E}_{q(\theta'|\bold{z}^N)}f_\theta(x)f_{\theta'}(x)}{2v^2}.
\end{align}
\begin{proof}
The proof is similar to \cite{masegosa2020second}.
First note that
\begin{align}
    &\frac{|y-\mathbb{E}_{q(\theta|\bold{z}^N)}f_\theta(x)|^2}{2v^2}\!+\lambda\frac{\mathrm{Var}_{\theta|\bold{z}^N}f_\theta(x)}{2v^2}\nonumber \\
    &=\mathbb{E}_{q(\theta|\bold{z}^N)}\frac{|y-f_\theta(x)|^2}{2v^2}\!+(\lambda-1)\frac{\mathrm{Var}_{\theta|\bold{z}^N}f_\theta(x)}{2v^2}
    \nonumber \\
    &=\mathbb{E}_{q(\theta|\bold{z}^N)}\frac{|y-f_\theta(x)|^2}{2v^2}\!+(\lambda-1)\frac{\mathbb{E}_{q(\theta|\bold{z}^N)}f_\theta^2(x)-\mathbb{E}_{q(\theta|\bold{z}^N)}\mathbb{E}_{q(\theta'|\bold{z}^N)}f_\theta(x)f_{\theta'}(x)}{2v^2}.
\end{align}
Based on this, we define the tandem loss as
\begin{align}
    L(z,\theta,\theta'):=\mathbb{E}_{q(\theta|\bold{z}^N)}\frac{|y-f_\theta(x)|^2}{2v^2}\!+(\lambda-1)\frac{\mathbb{E}_{q(\theta|\bold{z}^N)}f_\theta^2(x)-\mathbb{E}_{q(\theta|\bold{z}^N)}\mathbb{E}_{q(\theta'|\bold{z}^N)}f_\theta(x)f_{\theta'}(x)}{2v^2}.
\end{align}
Then by considering the prior $p(\theta,\theta')=p(\theta)p(\theta')$, using Theorem~\ref{THM:PAC}, we have
\begin{align}
    \mathbb{E}_{\nu(Z)q(\theta|\bold{z}^N)q(\theta'|\bold{z}^N)}L(Z,\theta,\theta')\leq& \frac{1}{N}\sum_{n=1}^N\mathbb{E}_{q(\theta|\bold{z}^N)q(\theta'|\bold{z}^N)}L(z_n,\theta,\theta')\nonumber \\
    &+\frac{\mathrm{KL}(q(\theta|\bold{z}^N)q(\theta'|\bold{z}^N)|p(\theta)p(\theta'))+\ln{\xi}^{-1}+\Omega_{p,\nu}(c,N)}{cN},
\end{align}
where
\begin{align}
    \Omega_{p,\nu}(c,N):=\ln\mathbb{E}_{p(\theta)p(\theta')}\mathbb{E}_{\nu(\bold{Z})^N}\mathrm{exp}[cN(\mathbb{E}_{\nu(Z)}L(Z,\theta,\theta')-\frac{1}{N}\sum_{n=1}^NL(Z_n,\theta,\theta'))].
\end{align}
Since $\mathrm{KL}(q(\theta|\bold{z}^N)q(\theta'|\bold{z}^N)|p(\theta)p(\theta'))=2\mathrm{KL}(q(\theta|\bold{z}^N)|p(\theta))$, by setting c=2c', we get the result.
\end{proof}
Thus, the constant $\Omega_{p,\nu}$ depends only on the setting of the problem. We optimize the right-hand side of Eq.\eqref{bound_ber} as the objective function.

Next, we discuss the relation between rBER and standard VI. The objective function of standard VI is
\begin{align}
-\mathbb{E}_{q(\theta|\bold{z}^N)}\ln N(y|f_\theta(x),v^2)
&=\frac{|y-\mathbb{E}_{q(\theta|\bold{z}^N)}f_\theta(x)|^2+ \mathrm{Var}_{\theta|\bold{z}^N}f_\theta(x)}{2v^2}+\frac{1}{2}\ln2\pi v^2.
\end{align}
Thus, we can interpret that the log loss of the Gaussian likelihood corresponds to the prediction risk and Bayesian excess risk. Since the prediction risk corresponds to the prediction performance, the standard VI implicitly controls the prediction performance and the Bayesian excess risk. Our rBER can be regarded as
\begin{align}
    \mathrm{rBER}(\lambda)=\frac{1}{2v^2}\frac{1}{N}\sum_{n=1}^N\left(\mathrm{PR}^{\mathrm{(2)}}(y_n|x_n,\bold{z}^N)+\lambda\mathrm{BER}^{\mathrm{(2)}}(y_n|x_n,\bold{z}^N)\right)+\frac{1}{2}\ln2\pi v^2+\frac{1}{N}\mathrm{KL}(q|p).
\end{align}
Thus, BER has a flexible weight for regularizing the uncertainty. Numerically, when $\lambda=0$, this corresponds to the setting where we simply optimize $\mathrm{PR}^{\mathrm{(2)}}(y|x,\bold{z}^N)$. This means we only consider the fitting performance. We numerically found that $\lambda=0$ results in large uncertainty due to the lack of regularization. When $\lambda=1$, we found that the uncertainty is underestimated.

Finally, we remark the relation between Bayesian excess risk and $\mathrm{Var}f_\theta(x)$. Since we focus on the log loss, thus we can consider the following type of objective function.
\begin{align}\label{eq_bayesian}
    \frac{1}{N}\sum_{n=1}^N\left[\frac{\mathbb{E}_{Q}|y_n\!-\!\mathbb{E}_{q(\theta|\bold{z}^N)}f_\theta(x_n)|^2}{2v^2}\!+\!\lambda\mathrm{BER}^\mathrm{log}(y_n|x_n,\bold{z}^N)\right]\!+\!\frac{\ln2\pi v^2}{2}+\frac{1}{N}\mathrm{KL}(q|p)\!,
\end{align}
where we use the Bayesian excess risk directly, instead of $\mathrm{Var}f_\theta(x)$. Note that from Appendix~\ref{app_general},  $\mathrm{BER}^{\mathrm{log}}(Y|x,\bold{z}^N)\leq \mathrm{Var}f_\theta(x)/v^2$ holds for the Gaussian likelihood. Thus, Eq.\eqref{eq_bayesian} and our BER behaves in a similar way. From the numerical point of view, implementing $\mathrm{Var}f_\theta(x)$ is easier than Eq.\eqref{eq_bayesian} since we simply calculate the variance of the prediction.

\section{Numerical experiments}\label{app:experiment}
In this section, we describe the detailed settings of the experiments. We also present the additional experimental results.
\subsection{Particle variational inference}
Since we applied our BER to particle variational inference (PVI), we briefly introduce the PVI and existing methods. In PVI, we use the empirical distribution $\rho(\theta)=\frac{1}{M}\sum_{i=1}^M\delta_{\theta_i}(\theta)$ as the posterior distribution. Here $\delta_{\theta_i}(\theta)$ is the Dirac distribution that has a mass at $\theta_i$. We refer to these $M$ samples as particles. PVI \cite{liu2016stein,DBLP:conf/iclr/WangRZZ19} approximates the posterior through iteratively updating the empirical distribution by interacting them with each other:
\begin{align}\label{PVI_intro}
\theta_i^{\mathrm{new}}\xleftarrow{} \theta_i^{\mathrm{old}}+\eta v_i(\{\theta_{i'}^{\mathrm{old}}\}_{i'=1}^M),
\end{align}
where $v(\{\theta\})$ is the update direction. Basically, $v$ is composed of the gradient term and the repulsion term  to enhance the divesity of the posterior distribution since we are often interested in the multi-modal information of the posterior distribution. For the update direction about $v$, see f-SVGD in \cite{DBLP:conf/iclr/WangRZZ19} and  VAR in \cite{futami2021loss} for details.

We follow the approach in VAR in \cite{futami2021loss}. They proposed using the gradient of the PAC-Bayesian bound for the update direction $v$ in Eq.\eqref{PVI_intro}. For example, VAR uses
\begin{align}
    &v_i=\partial_i \mathcal{F}(\{\theta_i\}_{i=1}^N),\\
   &\mathcal{F}(\{\theta_i\}_{i=1}^N):=-\frac{1}{NM}\sum_{i=1}^M\sum_{n=1}^N\left[\ln p(y_n|x_n,\theta_i)\!+R(y_n,x_n)\right]+\frac{1}{N}\mathrm{KL}(\rho_\mathrm{E}|\pi),
\end{align}
where $R$ is the repulsion term to enhance the diversity. See \cite{futami2021loss} for details. Following their setting, we consider using the following update direction
\begin{align}
&v_i=\partial_i \mathrm{rBER}(\lambda)\\
   &\scalebox{0.95}{$\displaystyle\mathrm{rBER}(\lambda)=\frac{1}{N}\sum_{n=1}^N \frac{|y_n\!-\!\mathbb{E}_{\rho(\theta)}f_\theta(x_n)|^2}{2v'^2}\!+\lambda\frac{\mathrm{Var}f_\theta(x_n)}{2v'^2}\!+\frac{\ln2\pi v^{'2}}{2}\!+\frac{1}{N}\mathrm{KL}(\rho(\theta)|p(\theta))$},
\end{align}
where $\rho(\theta)=\frac{1}{M}\sum_{i=1}^M\delta_{\theta_i}(\theta)$. We optimize $v'$ by gradient descent.

\subsection{Toy data experiments and Depth estimation}
For these experiments, we used the implementation in the previous work \cite{amini2020deep}. For the toy data experiments, we used the Adam optimizer with the stepsize $0.0001$ in the implementation of \cite{amini2020deep}. The number of ensembles is 5. We set other hyperparameters as the same as in \cite{amini2020deep}.

As for the depth estimation experiments, we used the same hyperparameter setting in  \cite{amini2020deep}. Here, the number of ensembles is 5.

\subsection{BNN regression for UCI dataset}
We used the same setting as the previous work \cite{DBLP:conf/iclr/WangRZZ19,futami2021loss}. We used the Adam optimizer with a learning rate of 0.004. We used a batch size of 100 and ran 500 epochs for the dataset size to be smaller than 1000. For a larger dataset, we used a batch size of 1000 and ran 3000 epochs. 

To calculate the PICP, we first calculate the $95\%$ prediction interval. We then calculate the number of the test data points included inside the prediction interval.

To calculate the MPIW, we calculated the mean of the prediction interval and normalized it by the maximum length of the test data point; $\max y_\mathrm{test}-\min y_\mathrm{test}$.

We show the additional results here. We show the result of $\mathrm{PAC}_E^2$ and the negative log-likelihood. We also show the result of the $\alpha$-divergence minimization. Following the previous work \cite{li2017dropout}, we considered the entropic loss for $\alpha$-divergence minimization. In the definition of \cite{li2017dropout}, $\alpha=0$ corresponds to the standard (exclusive) KL divergence, $\alpha=0.5$ corresponds to the Hellinger divergence, and $\alpha=1.0$ corresponds to the (inclusive) KL divergence, which is used in expectation propagation algorithm. We test on $\alpha=0.5, 1.0$ and $2.0$. The results are shown in Table~\ref{tab:bench_regm2} to \ref{tab:bench_regm5}. 

First, we found that $\alpha$-divergence minimization show similar behaviors to f-SVGD in RMSE and negative test log-likelihood. However, we found that $\alpha$-divergence minimization shows very large uncertainties since their PICP are much larger than 0.95 and their MPIW are larger than f-SVGD.

We found that $\mathrm{PAC}_E^2$ shows the similar to BER(0) measured in the negative log-likelihood, MPIW, and PICP. However, the prediction performance of $\mathrm{PAC}_E^2$ in RMSE is significantly worse than BER(0). This is because the objective function of $\mathrm{PAC}_E^2$ is the negative log-likelihood of the predictive distribution; thus, the performance in RMSE is not guaranteed. On the other hand, the objective function of BER(0) is based on the squared loss. Thus, it can show  performance in RMSE.

\begin{figure}
\centering
\begin{minipage}{1.0\textwidth}
\centering
\tblcaption{Benchmark results on test RMSE}
\label{tab:bench_regm2}
\resizebox{1.\textwidth}{!}{
\begin{tabular}{c|cccccccc}
\toprule
\fontsize{9}{7.2}\selectfont \bf{\multirow{1}{*}{Dataset}} & 
\multicolumn{8}{|c}{\fontsize{9}{7.2}\selectfont Avg. Test RMSE}  \\
 & \fontsize{9}{7.2}\selectfont f-SVGD&
  \fontsize{9}{7.2}\selectfont $\alpha=0.5$&
   \fontsize{9}{7.2}\selectfont $\alpha=1.0$&
    \fontsize{9}{7.2}\selectfont $\alpha=2.0$&
  \fontsize{9}{7.2}\selectfont VAR&
 \fontsize{9}{7.2}\selectfont $\mathrm{PAC}_E^2$ & \fontsize{9}{7.2}\selectfont BER($0.05$) &
 \fontsize{9}{7.2}\selectfont f-SVGD  \\
\hline
\fontsize{8}{7.2}\selectfont Concrete&
\fontsize{8}{7.2}\selectfont 4.33$\pm$0.8&
\fontsize{8}{7.2}\selectfont 4.51$\pm$0.8&
\fontsize{8}{7.2}\selectfont 4.67$\pm$0.7&
\fontsize{8}{7.2}\selectfont 4.98$\pm$0.6&
\fontsize{8}{7.2}\selectfont 4.30$\pm$0.7 &
\fontsize{8}{7.2}\selectfont 5.49$\pm$0.5 &
\fontsize{8}{7.2}\selectfont 4.47$\pm$0.6 &
\fontsize{8}{7.2}\selectfont 4.48$\pm$0.7  \\
\fontsize{8}{7.2}\selectfont Boston&
\fontsize{8}{7.2}\selectfont 2.54$\pm$0.50 &
\fontsize{8}{7.2}\selectfont 2.81$\pm$0.88 &
\fontsize{8}{7.2}\selectfont 2.87$\pm$0.80 &
\fontsize{8}{7.2}\selectfont 2.98$\pm$0.90 &
\fontsize{8}{7.2}\selectfont 2.53$\pm$0.50 &
\fontsize{8}{7.2}\selectfont 4.41$\pm$0.45 &
\fontsize{8}{7.2}\selectfont 2.53$\pm$0.50 &
\fontsize{8}{7.2}\selectfont 2.53$\pm$0.51   \\
\fontsize{8}{7.2}\selectfont Wine&
\fontsize{8}{7.2}\selectfont 0.61$\pm$0.04 &
\fontsize{8}{7.2}\selectfont 0.61$\pm$0.04 &
\fontsize{8}{7.2}\selectfont 0.61$\pm$0.04 &
\fontsize{8}{7.2}\selectfont 0.61$\pm$0.03 &
\fontsize{8}{7.2}\selectfont 0.61$\pm$0.04 &
\fontsize{8}{7.2}\selectfont 1.02$\pm$0.11 &
\fontsize{8}{7.2}\selectfont 0.64$\pm$0.04 &
\fontsize{8}{7.2}\selectfont 0.63$\pm$0.02   \\
\fontsize{8}{7.2}\selectfont Power&
\fontsize{8}{7.2}\selectfont 3.78$\pm$0.14 &
\fontsize{8}{7.2}\selectfont 3.78$\pm$0.11 &
\fontsize{8}{7.2}\selectfont 3.78$\pm$0.13 &
\fontsize{8}{7.2}\selectfont 3.80$\pm$0.11 &
\fontsize{8}{7.2}\selectfont 3.75$\pm$0.13 &
\fontsize{8}{7.2}\selectfont 5.24$\pm$0.45 &
\fontsize{8}{7.2}\selectfont 3.66$\pm$0.15&
\fontsize{8}{7.2}\selectfont 3.69$\pm$0.12 \\
\fontsize{8}{7.2}\selectfont Yacht&
\fontsize{8}{7.2}\selectfont 0.64$\pm$0.28 &
\fontsize{8}{7.2}\selectfont 0.56$\pm$0.26 &
\fontsize{8}{7.2}\selectfont 0.88$\pm$0.34 &
\fontsize{8}{7.2}\selectfont 0.99$\pm$0.68 &
\fontsize{8}{7.2}\selectfont 0.60$\pm$0.28 &
\fontsize{8}{7.2}\selectfont 0.71$\pm$0.41 &
\fontsize{8}{7.2}\selectfont 0.75$\pm$0.41 &
\fontsize{8}{7.2}\selectfont 0.78$\pm$0.48 \\
\fontsize{8}{7.2}\selectfont Protein&
\fontsize{8}{7.2}\selectfont 3.98$\pm$0.54 &
\fontsize{8}{7.2}\selectfont 4.05$\pm$0.13 &
\fontsize{8}{7.2}\selectfont 3.97$\pm$0.04 &
\fontsize{8}{7.2}\selectfont 4.02$\pm$0.09 &
\fontsize{8}{7.2}\selectfont 3.92$\pm$0.05 &
\fontsize{8}{7.2}\selectfont 7.96$\pm$0.10 &
\fontsize{8}{7.2}\selectfont 3.83$\pm$0.10 &
\fontsize{8}{7.2}\selectfont 3.85$\pm$0.05 \\
\bottomrule
\end{tabular}}
\end{minipage}
\begin{minipage}{1.0\textwidth}
\centering
\tblcaption{Benchmark results on test PICP and MPIW.}
\resizebox{1.\textwidth}{!}{
\begin{tabular}{c|ccccc}
\toprule
\fontsize{9}{7.2}\selectfont \bf{\multirow{1}{*}{Dataset}} &  \multicolumn{5}{c}{\fontsize{9}{7.2}\selectfont Avg. Test PICP and MPIW in parenthesis} \\
 & \fontsize{9}{7.2}\selectfont f-SVGD&
  \fontsize{9}{7.2}\selectfont VAR&
 \fontsize{9}{7.2}\selectfont $\mathrm{PAC}_E^2$ & \fontsize{9}{7.2}\selectfont BER($0$) &
 \fontsize{9}{7.2}\selectfont BER($0.05$) \\
\hline
\fontsize{8}{7.2}\selectfont Concrete&
\fontsize{8}{7.2}\selectfont 0.82$\pm$0.03 (0.13$\pm$0.00) &
\fontsize{8}{7.2}\selectfont 0.87$\pm$0.04 (0.16$\pm$0.01) &
\fontsize{8}{7.2}\selectfont 0.97$\pm$0.02 (0.57$\pm$0.04)&
\fontsize{8}{7.2}\selectfont 0.99$\pm$0.02 (0.50$\pm$0.04) &
\fontsize{8}{7.2}\selectfont {\bf 0.95$\pm$0.02} (0.25$\pm$0.02)  \\
\fontsize{8}{7.2}\selectfont Boston&
\fontsize{8}{7.2}\selectfont 0.63$\pm$0.07 (0.10$\pm$0.02) &
\fontsize{8}{7.2}\selectfont 0.76$\pm$0.05 (0.14$\pm$0.01) &
\fontsize{8}{7.2}\selectfont {\bf 0.94$\pm$0.04} (0.40$\pm$0.04) &
\fontsize{8}{7.2}\selectfont 0.97$\pm$0.01 (0.33$\pm$0.04) &
\fontsize{8}{7.2}\selectfont 0.92$\pm$0.04 (0.22$\pm$0.02)  \\
\fontsize{8}{7.2}\selectfont Wine&
\fontsize{8}{7.2}\selectfont 0.79$\pm$0.03 (0.32$\pm$0.05) &
\fontsize{8}{7.2}\selectfont 0.85$\pm$0.02 (0.39$\pm$0.06) &
\fontsize{8}{7.2}\selectfont 0.98$\pm$0.01 (1.06$\pm$0.01) &
\fontsize{8}{7.2}\selectfont 0.99$\pm$0.00 (1.61$\pm$0.00) &
\fontsize{8}{7.2}\selectfont {\bf 0.95$\pm$0.03} (0.32$\pm$0.15)   \\
\fontsize{8}{7.2}\selectfont Power&
\fontsize{8}{7.2}\selectfont 0.43$\pm$0.01 (0.07$\pm$0.00) &
\fontsize{8}{7.2}\selectfont 0.82$\pm$0.01 (0.15$\pm$0.00) &
\fontsize{8}{7.2}\selectfont 0.99$\pm$0.00 (0.57$\pm$0.02) &
\fontsize{8}{7.2}\selectfont 0.99$\pm$0.01 (0.81$\pm$0.01)&
\fontsize{8}{7.2}\selectfont 0.96$\pm$0.01 (0.37$\pm$0.01)    \\
\fontsize{8}{7.2}\selectfont Yacht&
\fontsize{8}{7.2}\selectfont 0.92$\pm$0.04 (0.02$\pm$0.01) &
\fontsize{8}{7.2}\selectfont 0.93$\pm$0.04 (0.04$\pm$0.01) &
\fontsize{8}{7.2}\selectfont 0.97$\pm$0.03 (0.07$\pm$0.00)&
\fontsize{8}{7.2}\selectfont {\bf 0.96$\pm$0.03} (0.10$\pm$0.01) &
\fontsize{8}{7.2}\selectfont {\bf 0.94$\pm$0.04} (0.08$\pm$0.01)\\
\fontsize{8}{7.2}\selectfont Protein&
\fontsize{8}{7.2}\selectfont 0.53$\pm$0.01 (0.24$\pm$0.01) &
\fontsize{8}{7.2}\selectfont 0.83$\pm$0.00 (0.58$\pm$0.01) &
\fontsize{8}{7.2}\selectfont 0.98 $\pm$0.00 (1.44$\pm$0.06)&
\fontsize{8}{7.2}\selectfont 1.0 $\pm$0.00 (5.04$\pm$0.01) &
\fontsize{8}{7.2}\selectfont {\bf 0.96$\pm$0.01} (0.86$\pm$0.00)\\
\bottomrule
\end{tabular}}
\end{minipage}

\begin{minipage}{0.85\textwidth}
\centering
\tblcaption{Benchmark results on test PICP and MPIW for $\alpha$-divergence minimization.}
\resizebox{1.\textwidth}{!}{
\begin{tabular}{c|ccc}
\toprule
\fontsize{9}{7.2}\selectfont \bf{\multirow{1}{*}{Dataset}} &  \multicolumn{3}{c}{\fontsize{9}{7.2}\selectfont Avg. Test PICP and MPIW in parenthesis} \\
 &    \fontsize{9}{7.2}\selectfont $\alpha=0.5$&
   \fontsize{9}{7.2}\selectfont $\alpha=1.0$&
    \fontsize{9}{7.2}\selectfont $\alpha=2.0$\\
\hline
\fontsize{8}{7.2}\selectfont Concrete&
\fontsize{8}{7.2}\selectfont 0.97$\pm$0.01 (0.41$\pm$0.03)&
\fontsize{8}{7.2}\selectfont 0.99$\pm$0.01 (0.46$\pm$0.03) &
\fontsize{8}{7.2}\selectfont 0.99$\pm$0.01 (0.52$\pm$0.04)  \\
\fontsize{8}{7.2}\selectfont Boston&
\fontsize{8}{7.2}\selectfont 0.98$\pm$0.01 (0.51$\pm$0.08) &
\fontsize{8}{7.2}\selectfont 0.99$\pm$0.01 (0.54$\pm$0.08) &
\fontsize{8}{7.2}\selectfont 0.99$\pm$0.01 (0.57$\pm$0.09)  \\
\fontsize{8}{7.2}\selectfont Wine&
\fontsize{8}{7.2}\selectfont 0.94$\pm$0.02 (0.54$\pm$0.08) &
\fontsize{8}{7.2}\selectfont 0.95$\pm$0.02 (0.58$\pm$0.09) &
\fontsize{8}{7.2}\selectfont 0.96$\pm$0.01 (0.63$\pm$0.09)   \\
\fontsize{8}{7.2}\selectfont Power&
\fontsize{8}{7.2}\selectfont 0.98$\pm$0.00 (0.45$\pm$0.01) &
\fontsize{8}{7.2}\selectfont 0.99$\pm$0.00 (0.47$\pm$0.01) &
\fontsize{8}{7.2}\selectfont 1.00$\pm$0.00 (0.50$\pm$0.01)
\\
\fontsize{8}{7.2}\selectfont Yacht&
\fontsize{8}{7.2}\selectfont 0.99$\pm$0.01 (0.16$\pm$0.01) &
\fontsize{8}{7.2}\selectfont 1.00$\pm$0.00 (0.30$\pm$0.04) &
\fontsize{8}{7.2}\selectfont 1.00$\pm$0.00 (0.60$\pm$0.06)\\
\fontsize{8}{7.2}\selectfont Protein&
\fontsize{8}{7.2}\selectfont 0.99 $\pm$0.00 (1.13$\pm$0.02)&
\fontsize{8}{7.2}\selectfont 0.99 $\pm$0.00 (1.12$\pm$0.03) &
\fontsize{8}{7.2}\selectfont 0.99$\pm$0.01 (1.15$\pm$0.04)\\
\bottomrule
\end{tabular}}
\end{minipage}

\begin{minipage}{1.0\textwidth}
\centering
\tblcaption{Benchmark results on negative test log-likelihood}
\label{tab:bench_regm5}
\resizebox{1.\textwidth}{!}{
\begin{tabular}{c|cccccccc}
\toprule
\fontsize{9}{7.2}\selectfont \bf{\multirow{1}{*}{Dataset}} & 
\multicolumn{8}{|c}{\fontsize{9}{7.2}\selectfont Avg. negative test log likelihood}\\
 & \fontsize{9}{7.2}\selectfont f-SVGD&
   \fontsize{9}{7.2}\selectfont $\alpha=0.5$&
   \fontsize{9}{7.2}\selectfont $\alpha=1.0$&
    \fontsize{9}{7.2}\selectfont $\alpha=2.0$&
  \fontsize{9}{7.2}\selectfont VAR&
 \fontsize{9}{7.2}\selectfont $\mathrm{PAC}_E^2$ & \fontsize{9}{7.2}\selectfont BER($0$) &
 \fontsize{9}{7.2}\selectfont BER($0.05$)  \\
\hline
\fontsize{8}{7.2}\selectfont Concrete&
\fontsize{8}{7.2}\selectfont -2.85$\pm$0.15&
\fontsize{8}{7.2}\selectfont -2.79$\pm$0.17&
\fontsize{8}{7.2}\selectfont -2.82$\pm$0.08&
\fontsize{8}{7.2}\selectfont -2.95$\pm$0.05&
\fontsize{8}{7.2}\selectfont -2.81$\pm$0.06 &
\fontsize{8}{7.2}\selectfont -3.16$\pm$0.03 &
\fontsize{8}{7.2}\selectfont -3.50$\pm$0.03 &
\fontsize{8}{7.2}\selectfont -3.06$\pm$0.05  \\
\fontsize{8}{7.2}\selectfont Boston&
\fontsize{8}{7.2}\selectfont -2.34$\pm$0.31 &
\fontsize{8}{7.2}\selectfont -2.39$\pm$0.16 &
\fontsize{8}{7.2}\selectfont -2.43$\pm$0.12 &
\fontsize{8}{7.2}\selectfont -2.47$\pm$0.09 &
\fontsize{8}{7.2}\selectfont -2.34$\pm$0.24 &
\fontsize{8}{7.2}\selectfont -2.61$\pm$0.08 &
\fontsize{8}{7.2}\selectfont -2.55$\pm$0.05 &
\fontsize{8}{7.2}\selectfont -2.38$\pm$0.16  \\
\fontsize{8}{7.2}\selectfont Wine&
\fontsize{8}{7.2}\selectfont -0.89$\pm$0.08 &
\fontsize{8}{7.2}\selectfont -0.87$\pm$0.07 &
\fontsize{8}{7.2}\selectfont -0.88$\pm$0.05 &
\fontsize{8}{7.2}\selectfont -0.93$\pm$0.03 &
\fontsize{8}{7.2}\selectfont -0.90$\pm$0.06 &
\fontsize{8}{7.2}\selectfont -1.26$\pm$0.02 &
\fontsize{8}{7.2}\selectfont -1.84$\pm$0.03 &
\fontsize{8}{7.2}\selectfont -1.08$\pm$0.03   \\
\fontsize{8}{7.2}\selectfont Power&
\fontsize{8}{7.2}\selectfont -2.75$\pm$0.03 &
\fontsize{8}{7.2}\selectfont -2.73$\pm$0.02 &
\fontsize{8}{7.2}\selectfont -2.74$\pm$0.01 &
\fontsize{8}{7.2}\selectfont -2.80$\pm$0.03 &
\fontsize{8}{7.2}\selectfont -2.80$\pm$0.03 &
\fontsize{8}{7.2}\selectfont -3.17$\pm$0.03 &
\fontsize{8}{7.2}\selectfont -3.95$\pm$0.04&
\fontsize{8}{7.2}\selectfont -2.86$\pm$0.01  \\
\fontsize{8}{7.2}\selectfont Yacht&
\fontsize{8}{7.2}\selectfont -0.81$\pm$0.67 &
\fontsize{8}{7.2}\selectfont -0.92$\pm$0.27 &
\fontsize{8}{7.2}\selectfont -0.77$\pm$0.17 &
\fontsize{8}{7.2}\selectfont -1.58$\pm$0.17 &
\fontsize{8}{7.2}\selectfont -0.87$\pm$0.38 &
\fontsize{8}{7.2}\selectfont -0.81$\pm$0.11 &
\fontsize{8}{7.2}\selectfont -1.62$\pm$0.17 &
\fontsize{8}{7.2}\selectfont -1.46$\pm$0.30\\
\fontsize{8}{7.2}\selectfont Protein&
\fontsize{8}{7.2}\selectfont -2.70$\pm$0.00 &
\fontsize{8}{7.2}\selectfont -2.98$\pm$0.02 &
\fontsize{8}{7.2}\selectfont -2.90$\pm$0.01 &
\fontsize{8}{7.2}\selectfont -2.95$\pm$0.02 &
\fontsize{8}{7.2}\selectfont -2.84$\pm$0.00 &
\fontsize{8}{7.2}\selectfont -3.30$\pm$0.00 &
\fontsize{8}{7.2}\selectfont -4.45$\pm$0.02 &
\fontsize{8}{7.2}\selectfont -2.94$\pm$0.00 \\
\bottomrule
\end{tabular}}
\end{minipage}
\end{figure}

Next, we evaluated how the RMSE, PICP, MPIW, and negative log-likelihood behave by changing $\lambda$ in BER. We show the results in Fig.\ref{fig_uci_con} and \ref{fig_uci_wine}. We confirmed that the prediction performance measured in RMSE does not depend on the choice of $\lambda$. On the other hand, other measures depend on $\lambda$ significantly. The ideal PICP is 0.95. Thus, we should choose $\lambda$ by cross-validation. We also found that even $\lambda=1.$ correspond to the standard VI. It underestimates the PICP.
\begin{figure}[tb!]
 \centering
 \includegraphics[width=0.9\linewidth]{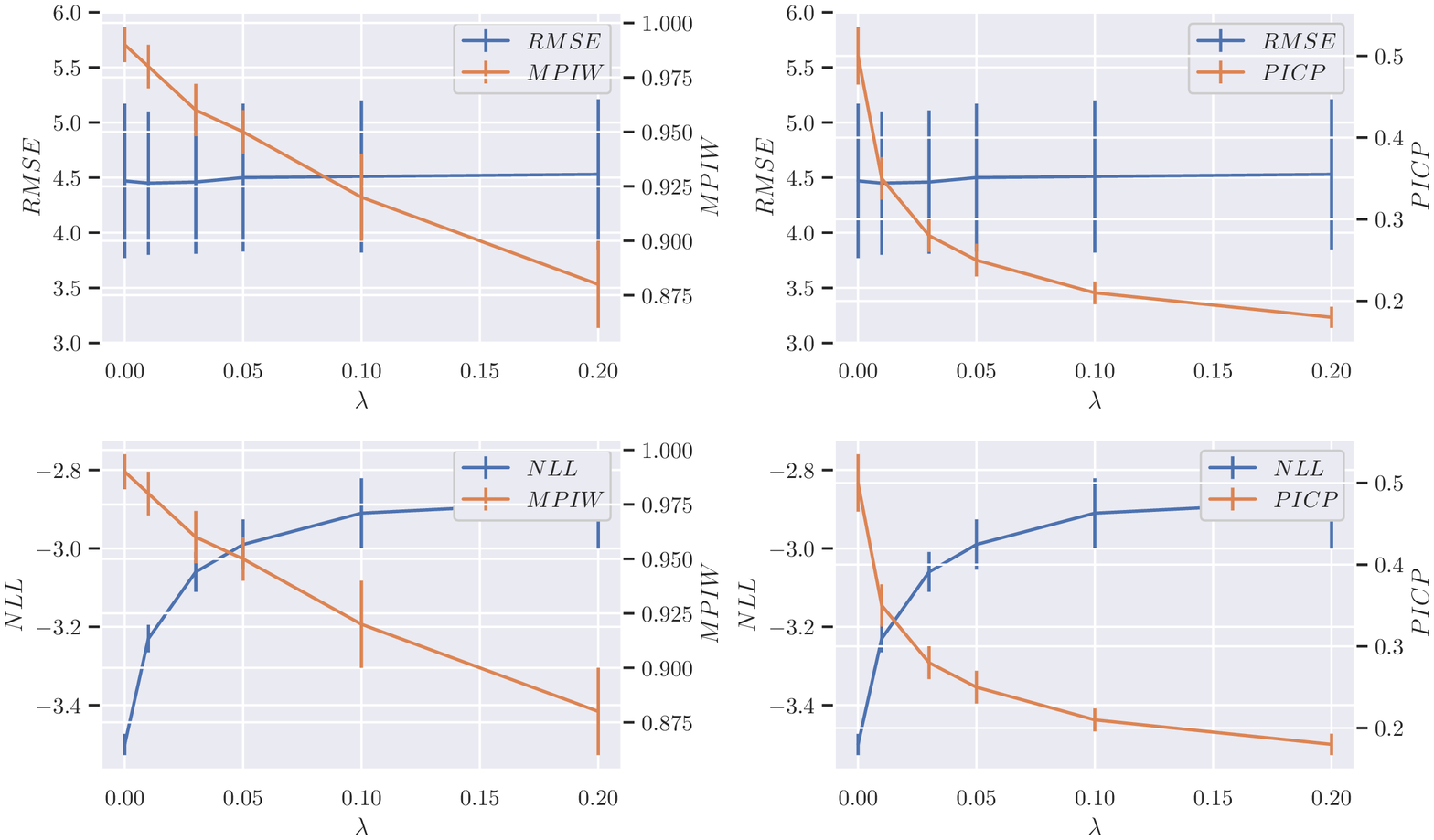}
 \caption{Concrete data in UCI dataset.}
\label{fig_uci_con}
\end{figure}

\begin{figure}[tb!]
 \centering
 \includegraphics[width=0.9\linewidth]{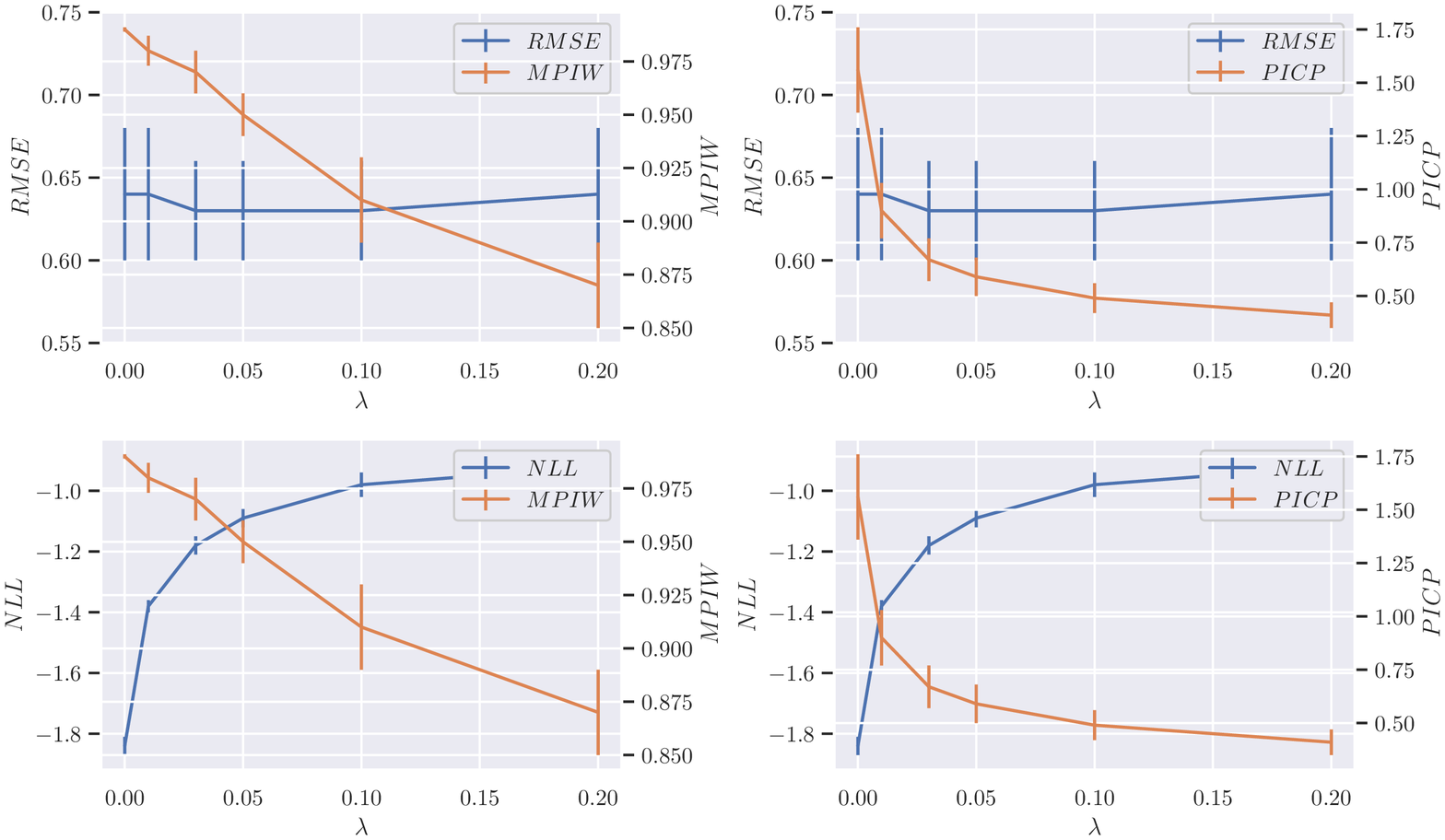}
 \caption{Wine data in UCI dataset.}
\label{fig_uci_wine}
\end{figure}

\subsection{Contextual bandit tasks}
Here we explain the setting of the task. Our experiments follow the setting in \cite{DBLP:conf/iclr/WangRZZ19,futami2021loss}. Denote the context set as $\mathcal{S}$. For each time step $t$, an agent recieves context $s_t\in \mathcal{S}$ from the environment. The agent choose action $a_t\in\{1,\ldots,A\}$ based on the context $s_t$ and get a reward $r_{a_t,t}$. We would like to  minimize the pseudo regret
\begin{align}
    R_T=\max_{\phi:\mathcal{S}\to\{1,\ldots,A\}}\mathbb{E}\left[\sum_{t=1}^Tr_{g(s_t),t}-\sum_{t=1}^Tr_{a_t,t}\right],
\end{align}
where $\phi$ maps the context to the action. We consider a prior $\mu_{s,i,0}$ over a reward of context $s$ and action $i$. Then, we update the prior to a posterior distribution using the observed reward.  Following the previous work, we use Thompson sampling to select the action as
\begin{align}
    r_t\in\argmax_{i=\{1,\ldots,K\}}\hat{r}_{i,t},\quad \hat{r}_{i,t}\sim \mu_{s,i,t}.
\end{align}

We consider a neural network regression model following the previous work  \cite{DBLP:conf/iclr/WangRZZ19,futami2021loss}, where the input is the context, and the output is the $K$-dimensional action. We place a prior distribution over the parameters of the network. We approximate the posterior distribution of the neural network parameters by PVI. All the hyperparameters are precisely the same as in the previous work \cite{DBLP:conf/iclr/WangRZZ19}.

\end{document}